\documentclass[twoside,11pt,preprint]{article}

\usepackage{jmlr2e}
\usepackage{amsmath,amssymb}
\usepackage{xcolor}
\usepackage{comment}

\newcommand{\Tr}{\mathrm{Tr}}
\newcommand{\E}[1]{{\mathbb E}\left[#1\right]}
\newcommand{\eps}{\varepsilon}
\newcommand{\R}{\mathbb R}

\newcommand{\SE}[1]{{\color{black}{#1}}}

\newcommand{\MT}[1]{{\color{violet}{#1}}}

\DeclareMathOperator*{\argmin}{arg\,min}
\DeclareMathOperator*{\supp}{supp}

\newtheorem{assumption}{Assumption}

\firstpageno{1}

\usepackage{lastpage}
\jmlrheading{24}{2023}{1-\pageref{LastPage}}{3/22; Revised
	11/23}{11/23}{22-0303}{Stephan Eckstein, Armin Iske, Mathias Trabs}
\ShortHeadings{Dimensionality Reduction and Kernel Regression}{Eckstein, Iske and Trabs}
\begin{document}

\title{Dimensionality Reduction and Wasserstein Stability for Kernel Regression}

\author{\name Stephan Eckstein \email seckstein@ethz.ch \\
       \addr Department of Mathematics\\
       ETH Z\"{u}rich\\
       Zurich, CH
       \AND
       \name Armin Iske \email armin.iske@uni-hamburg.de \\
       \addr Department of Mathematics\\
       Universit\"{a}t Hamburg\\
       Hamburg, DE
       \AND
       \name Mathias Trabs \email trabs@kit.edu \\
	   \addr Department of Mathematics\\
	   Karlsruhe Institute of Technology\\
	   Karlsruhe, DE}

\editor{Miguel Carreira-Perpinan}
\maketitle

\begin{abstract}
In a high-dimensional regression framework, we study consequences of the naive two-step procedure where first the dimension of the input variables is reduced and second, the reduced input variables are used to predict the output variable with kernel regression. 
In order to analyze the resulting regression errors, a novel stability result for kernel regression with respect to the Wasserstein distance is derived. This allows us to bound errors that occur when perturbed input data is used to fit the regression function. We apply the general stability result to principal component analysis (PCA). Exploiting known estimates from the literature on both principal component analysis and kernel regression, we deduce convergence rates for the two-step procedure. The latter turns out to be particularly useful in a semi-supervised setting.
\end{abstract}


\section{Introduction}

	In data analysis, it is natural to first understand and reduce the explanatory data before fitting a regression \citep[see, e.g.,][]{abdul2005principal, adragni2009sufficient, andras2017high, fradkin2003experiments, hua2007modeling}. Aside from principal component analysis \citep[PCA; see, e.g.,][]{wold1987principal}, a variety of different approaches, like UMAP \citep{mcinnes2018umap}, MDS \citep{cox2008multidimensional}, autoencoders \citep{ng2011sparse, tolstikhin2018wasserstein, wang2016auto}, SIR \citep{li1991sliced}, KPCA \citep{scholkopf1997kernel}, random projections \citep{bingham2001random}, etc., have become increasingly popular lately. There is a similarly large variety of regression and function estimation methods (like neural networks, kernel methods, decision forests, etc.). This paper aims to shed some light on the estimation errors that occur when first applying a dimensionality reduction method before fitting a regression. We focus on two frequently used and studied methods: PCA and kernel regression.

\medskip
If $X$ is the explanatory data and $Y$ is an output variable in a standard $L^2$-regression framework, the goal is to find a regression function $g$ such that 
\[\mathbb{E}[|g(X) - Y|^2]\]
is minimized.
The optimal regression function $g$ given perfect knowledge of the underlying distributions is given by $g(X) = \mathbb{E}[Y | X]$.
When first applying a dimensionality reduction map $P$ on $X$, the goal is instead to find a regression function $f$ minimizing
\[\mathbb{E}[|f(P(X)) - Y|^2].\]

On the one hand reducing the dimension before performing a regression has several advantages, especially in view of
computational aspects and interpretability of the occurring features and the obtained regression function \citep[see, e.g.,][]{chipman2005interpretable, GuillemardIske2017, kim2005dimension}. An aspect that is of crucial importance in any regression framework is \textsl{sample efficiency}, i.e., how much sample data is required to obtain good approximations of the true regression function. Usually, sample efficiency tends to deteriorate in higher dimensions, which is an instance of the curse of dimensionality. Studying sample efficiency in combination with dimensionality reduction is a focus of the analysis in this paper.

One the other hand, dimensionality reduction decreases the predictive power of the explanatory data. Indeed, in terms of information content, it is clear that for any non-invertible transformation $P$, $P(X)$ contains at most as much information as $X$ does about $Y$ \citep{adragni2009sufficient}. However, in many cases where dimensionality reduction is applied, the implicit assumption is that the reduced data $P(X)$ retains almost all relevant information encoded in $X$, and thus the fit $f(P(X)) \approx Y$ should be almost as good as $g(X) \approx Y$. For PCA, the structural error $f\circ P - g$ can be controlled via the eigenvalues of the covariance matrix of $X$, see Proposition \ref{prop:procedureerror}.

\medskip
In sample based regression, only finitely many i.i.d.~copies of $(X, Y)$ \[
(X_1, Y_1), \dots, (X_n, Y_n)
\] are given without knowing the true distribution of $(X, Y)$. For kernel regression, the convergence of the sample based regression functions for $n \rightarrow \infty$ is a well studied problem \citep[see, e.g.,][]{caponnetto2007optimal, mendelson2010regularization, steinwart2009optimal, wu2006learning}. When including a prior dimensionality reduction step, however, new errors occur that need to be controlled. 

Let $P_n$ denote the estimator of the dimensionality reduction map using $X_1, \dots, X_n$ and let $f_n$ denote the estimated regression function such that $f_n(P_n(X_i)) \approx Y_i$ for $i=1, \dots, n$. Comparing $f_n \circ P_n$ to $f \circ P$, we observe two different error sources: First, obviously both $f_n$ and $P_n$ include the statistical error introduced by sampling. Second however, a crucial error that is introduced for $f_n$ is that already the input to its estimation procedure, $P_n(X_1), \dots, P_n(X_n)$, differs \emph{in distribution} from the data $P(X)$ which is used to reconstruct $f$. This means, the sample points $P_n(X_1), \dots, P_n(X_n)$ which are used to estimate $f_n$ are not merely i.i.d.~copies of $P(X)$ which is used to characterize $f$. Thus, the question of how close $f_n$ is to $f$ requires a suitable stability of kernel regression. In Theorem \ref{thm:stab}, we obtain a stability result which can control the difference between two kernel estimators when the difference of the underlying distributions is measured in the first Wasserstein distance. 
Subsequently, in Theorem \ref{thm:generaldimred}, the stability result is used to obtain bounds on the joint procedure for general dimensionality reduction methods, which assumes that we can control both the operator norm $\|P_n - P\|_{\rm op}$ and also the regression error for the regression on the low-dimensional space. These bounds are detailed out for the case of PCA in Theorem \ref{thm:overallerror}. 

We emphasize that the subspaces spanned by estimated PCA and true PCA map may have entirely different support, and hence it is crucial that the stability result in Theorem \ref{thm:stab} works with a metric distance like the Wasserstein distance, while stability results from the literature that work with the total variation norm or similar distances \citep[see, e.g.,][]{bousquet2002stability,christmann2007consistency,christmann2018total,xu2009robustness} cannot be applied here. While we focus on PCA as a standard and well studied dimensionality reduction approach, Theorem \ref{thm:stab} and the results in Section \ref{subsec:generaldimred} allow also to study the impact of more advanced methods. These are especially necessary if the explanatory variables cannot be described by a linear subspace, but, for instance, by a submanifold.

Intuitively, one expects that first reducing the dimension of the data and 
then performing a regression on the low-dimensional data should be as 
sample efficient as starting with the low-dimensional data directly. As 
just discussed, this intuition ignores the added difficulty of the 
estimation error for the dimensionality reduction map, and in particular that 
this error is propagated by the stability of the low-dimensional 
regression. We obtain error estimates in Theorem \ref{thm:overallerror} and 
Corollary \ref{cor:errorsummary} for the overall procedure of PCA and 
kernel regression where the errors coming from this source are indeed 
restrictive, and thus the overall error estimates might be worse than 
starting directly with low-dimensional data. Thus, the analysis in this 
paper suggests that two considerations for dimensionality reduction in terms of 
sample efficiency have to be taken into account: On the one hand, the 
actual regression on the low-dimensional space is usually faster (see 
Section \ref{subsec:kerneltransformation} for a quantification thereof), 
but on the other hand, a new error through estimation of the dimension 
reduction map is introduced. This comparison is discussed in more detail in 
Section~\ref{sec:PCAkernel}. We conclude that the two-step approach is 
particularly useful if the dimension of the input data is large and the 
kernel is not too regular. These theoretical findings are 
confirmed in a numerical illustration.

Another considerable benefit is in semi-supervised learning \citep[see, e.g.,][]{zhu2009introduction, zhu2005semi}: If a larger sample for $X$ than for $(X, Y)$ pairs is given, one can use the whole $X$ sample for the estimation of the dimensionality reduction map.
If the sample for $X$ (compared to $(X, Y)$ sample pairs) is sufficiently large, Proposition \ref{prop:mn} shows that the error arising from the estimation of the PCA map is no longer restrictive for the overall error of the PCA and kernel regression procedure, and thus the sample efficiency of the low-dimensional kernel regression is recovered.

\medskip
We emphasize that the studied setting in this paper differs from the frequently studied question of \textsl{simultaneously} estimating a dimensionality reduction map $P$ and an estimator $f$ such that $f(P(X)) \approx Y$, which is for instance the case in multi-index models and related fields \citep[see, e.g.,][]{alquier2013sparse, devore2011approximation, gaiffas2007optimal, gu2013smoothing, lin2006component, raskutti2012minimax}. The method of simply applying an off-the-shelf dimensionality reduction method like PCA certainly leads to worse performance on regression tasks, as the dimensionality reduction map is estimated using just the $X$-data, ignoring the output variable $Y$. Then again, such an approach has other benefits, like a clear interpretation of the reduced data and being able to use the same reduced data for different regression tasks.

Finally, we mention that related objectives to this paper have recently been studied from different perspectives, like dimensionality reduction for kernel regression based on partitioning the space \citep{hamm2021adaptive,hamm2022intrinsic}, or related stability questions by studying out-of-distribution performance of kernel methods \citep{canatar2021out, cui2021generalization}.

\medskip
The remainder of the paper is structured as follows: We introduce the 
kernel regression setting in Section \ref{subsec:kernelsetting} and analyze 
the basic influence that dimension and dimensionality reduction methods have on 
sample efficiency in Section \ref{subsec:kerneltransformation}. The 
Wasserstein stability result for kernel regression is given in Section 
\ref{subsec:kernelstability}. Section \ref{sec:secmain} studies the 
combined procedure of dimensionality reduction and kernel regression. Section \ref{subsec:generaldimred} focuses on general dimensionality reduction methods. Section \ref{sec:PCA} 
states basic results from the literature on PCA and Section 
\ref{sec:PCAkernel} contains the results on the estimation error when PCA is combined with kernel regression. Section 
\ref{sec:examples} gives numerical examples illustrating the procedure. All proofs are postponed to Section~\ref{sec:proofs}.
\medskip

Throughout this paper, we denote by $|\cdot|$ the Euclidean norm. For a bounded linear operator $P$ on a Hilbert space, we denote by $\|P\|_{\rm op}$ the operator norm and $\|P\|_2$ the Hilbert-Schmidt norm, and similarly we denote by $\|P\|_{\rm op}$ the operator norm for a map $P$ mapping one normed space into another. 

	\section{Kernel Regression: Dimension dependence, Convergence Rates, and Stability}
\label{sec:kernel}
This section gives results related to kernel regression.
First, in Subsection \ref{subsec:kernelsetting} we introduce the kernel regression framework and state results from the literature on convergence rates. In Subsection \ref{subsec:kerneltransformation}, we give basic intuition and results on how kernel regression behaves under simple transformations of input measures, namely linear dimensionality reduction or inclusion of independent noise. And finally, Subsection \ref{subsec:kernelstability} studies stability of the optimal kernel regression function when the input data is perturbed.


\subsection{Kernel Regression Setting}
\label{subsec:kernelsetting}
We are given an input space $\mathcal{X} \subseteq \mathbb{R}^D$ and an output space $\mathcal{Y} \subseteq [-M, M]$ for some $M>0$. The learning problem we are interested in is governed by a probability distribution $\rho$ on $\mathcal{Z} = \mathcal{X} \times \mathcal{Y}$, where the goal is to find the optimal predictor of the output variables, given the input variables. This means, the goal is to solve
\[
f_\rho := \argmin_{f} \int (f(x) - y)^2 \,\rho(dx, dy),
\]
where the $\argmin$ is taken over all measurable functions mapping $\mathcal{X}$ to $\mathcal{Y}$. Writing $\rho(dx, dy) = \rho_X(dx) \rho(dy | x)$, the solution is given by the conditional mean
\[
f_\rho(x) = \int y \,\rho(dy | x) = \mathbb{E}[Y | X = x].
\]
However, $\rho$ is unknown, and only finitely many sample pairs 
\[(X_1, Y_1), (X_2, Y_2), \dots, (X_n, Y_n) \in \mathcal{Z}\]
are observed, which are independent and identically distributed according to $\rho$. To approximate $f_\rho$ using these finitely many observations, we introduce the regularized kernel regression problem
\[
f_{\lambda, n} := \argmin_{f \in \mathcal{H}} \frac{1}{n} \sum_{i=1}^n (f(X_i) - Y_i)^2 + \lambda \|f\|^2_{\mathcal{H}},
\]
where $1 \geq \lambda > 0$ is a regularization parameter and $\mathcal{H}$ is a \emph{reproducing kernel Hilbert space} \citep[see, e.g.,][]{cucker2002mathematical, Iske2018} of functions mapping $\mathcal{X}$ to $\mathbb{R}$. We use the notation $K\colon \mathcal{X} \times \mathcal{X} \rightarrow \mathbb{R}$ for the kernel associated with $\mathcal{H}$, and always assume that the kernel is normalized to $\sup_{x \in \mathcal{X}} K(x, x) = 1$. We further introduce
\begin{align}
\label{eq:solutionflambdarho}
f_{\lambda, \rho} :=& \argmin_{f \in \mathcal{H}} \int (f(x)-y)^2 \,\rho(dx, dy) + \lambda \|f\|^2_{\mathcal{H}},\\
=&\argmin_{f \in \mathcal{H}} \int (f(x)-f_\rho(x))^2 \,\rho_X(dx) + \lambda \|f\|^2_{\mathcal{H}}\label{eq:interpolation}
\end{align}
which is called the population counterpart to $f_{\lambda, n}$. If $\rho_n$ denotes the empirical distribution of the given sample, then we have $f_{\lambda, \rho_n} = f_{\lambda, n}$.
Defining $T_\rho\colon \mathcal{H} \rightarrow \mathcal{H}$ and $g_\rho \in \mathcal{H}$ by 
\begin{align*}
T_\rho f (x) &:= \int K(x, u) f(u) \,\rho_X(du),\\
g_\rho (x) &:= \int K(x, u) y \,\rho(du, dy)
\end{align*}
we recall that the solution of \eqref{eq:solutionflambdarho} is given by \citep{caponnetto2007optimal}
\[
f_{\lambda, \rho} = (T_\rho + \lambda I)^{-1} g_\rho.
\]
Finally, we introduce the clipped estimator 
\[\boldsymbol{f}_{\lambda, \rho} := -M \lor (M \land f_{\lambda, \rho}),
\] and similarly $\boldsymbol{f}_{\lambda, n}$. The purpose of clipping is that, since $f_\rho$ only takes values in $[-M, M]$, it holds $|\boldsymbol{f}_{\lambda, \rho} - f_\rho| \leq |f_{\lambda, \rho} - f_\rho|$ and in certain aspects $\boldsymbol{f}_{\lambda, \rho}$ is more tractable than $f_{\lambda, \rho}$  \citep[see also the discussion in][and references therein]{steinwart2009optimal}.

We will require an analogue of $T_\rho$ acting on $L^2$. To this end, for a probability measure $\mu$ on $\mathcal{X}$, we define the kernel integral operator $L_{K, \mu}\colon L^2(\mu) \rightarrow L^2(\mu)$ by
\[
L_{K, \mu} f(x) := \int K(x, x') f(x') \,\mu(dx'),
\]
which is compact and self-adjoint, cf.~\citet[Chapter III]{cucker2002mathematical}.

To study the approximation of $\|f_{\lambda, n} - f_\rho\|_{L^2(\rho_X)}$, two key parameters were identified in the literature to describe the convergence rate in $n$ for a suitable choice of $\lambda$:
\begin{assumption}
	\label{ass:alphabeta}$\,$
	\begin{itemize}
		\item[(i)] Let $\alpha \in (0, 1)$ such that there exists a constant $C_{\alpha} > 0$ so that the sequence of non-increasing eigenvalues $\sigma_1 \geq \sigma_2 \geq ...$ of $L_{K, \rho_X}$ satisfies $\sigma_n \leq C_{\alpha} n^{-\frac{1}{\alpha}}$.
		\item[(ii)] Let $\beta \in (0, 1]$ such that there exists a constant $C > 0$ so that 
		\begin{equation}\label{eq:assbeta}
		\|f_{\lambda, \rho} - f_{\rho}\|^{2}_{L^2(\rho_X)}+\lambda\|f_{\lambda,\rho}\|^2_\mathcal H \leq C \lambda^\beta.
		\end{equation}
	\end{itemize} 
\end{assumption}
Roughly, the parameter $\alpha$ describes the complexity of the measure $\rho_X$ in terms of the kernel $K$ while  $\beta$ measures the approximation error of the function $f_\rho$ with functions in $\mathcal H$. We briefly discuss $\alpha$ and $\beta$ as follows.

\begin{remark}\label{rem:sobolev}
	Since $\int K(x,x)^2\rho_X(dx)\le 1$, $L_{K,\rho_X}$ is compact and we have $\alpha\le 1$ \citep[Theorem 4.27]{SteinwartChristmann2008}. If the kernel is $m$-times continuously differentiable and $\rho_X$ is the uniform distribution on the Euclidean unit ball in $\mathbb{R}^D$, where $D < 2m$, then Assumption~\ref{ass:alphabeta}(i) is satisfied for $\alpha = \frac{D}{2m}$, see \citet[page 3]{steinwart2009optimal}. In particular, regular kernels lead to smaller $\alpha$ and a faster decay of $(\sigma_n)$. On the other hand, the decay suffers from large dimensions.
	
	In view of \eqref{eq:interpolation}, Assumption~\ref{ass:alphabeta}(ii) quantifies the approximation quality of $f_\rho$ in a $\mathcal H$-ball. The function $A_2(\lambda) = \|f_{\lambda, \rho} - f_{\rho}\|^{2}_{L^2(\rho_X)}+\lambda\|f_{\lambda,\rho}\|^2_\mathcal H$ occurring in Assumption~\ref{ass:alphabeta}(ii) is also called the \textsl{approximation error function}. It is directly related to the interpolation properties of the RKHS, see \citet[Chapter 5.6]{SteinwartChristmann2008}. For instance, if $\mathcal H$ is a Sobolev space of order $m$, then the kernel is $m$-times differentiable and we can choose $\beta=\frac{2k}{2m}$ for $k$-Sobolev regular functions $f_\rho$ where $k\le m$ \citep[cf.][Chapter~11]{wendland2004scattered}. In particular, for a fixed regularity of the regression function $f_\rho$ very regular kernels lead to small $\beta$.
\end{remark}

We will require the following assumption for the kernel $K$.
\begin{assumption}
	\label{ass:kernel}
	Assume the form $K(x_1, x_2) = \phi(|x_1-x_2|)$ for kernel $K$, where $\phi\colon \mathbb{R}_+ \rightarrow \mathbb{R}$ satisfies the growth condition $\phi(0) - \phi(r) \leq \frac{L^2}{2} r^2$.
\end{assumption}
This assumption immediately implies the bound \citep[cf.][Section~8.4.2]{Iske2018} \begin{equation}\label{eq:KLip}
\|K(x_1, \cdot) - K(x_2, \cdot)\|_{\mathcal{H}} = (2 (\phi(0) - \phi(|x_1-x_2|)))^{1/2} \leq L |x_1 - x_2|.
\end{equation}
In particular, all $f\in\mathcal H$ are bounded and Lipschitz continuous with  $\|f\|_\infty\le \|f\|_{\mathcal H}$ (recalling that $\phi(0) = K(x, x) = 1$) and  $|f(x)-f(x')|=|\langle K(x,\cdot)-K(x',\cdot),f\rangle_{\mathcal H}|\le L\|f\|_{\mathcal H}|x-x'|$. Throughout, we denote the Lipschitz norm of a Lipschitz continuous function $g\colon\mathcal X\to\mathcal Y$ by
	\[
	\|g\|_{\mathrm{Lip}}:=\sup_{\substack{x_1,x_2\in\mathcal X:\\x_1\neq x_2}}\frac{|g(x_1)-g(x_2)|}{|x_1-x_2|}.
	\]

Assumption \ref{ass:kernel} is satisfied for frequently used choices like the Gaussian kernel $K(x, y) = \exp(-|x-y|^2)$ or various compactly supported radial kernels like $K(x, y) = (1-|x-y|^2)_+^l$ for $l\in \mathbb{N}$ \cite[Example 2]{wu1995compactly}, see also \citet[Table 4.1]{zhu2012compactly} or \cite{wendland2004scattered} for more examples. In general, any function $\phi$ that is twice continuously differentiable, satisfies $\phi'(0) = 0$ and has bounded second derivative fulfills Assumption \ref{ass:kernel}, as can be seen by using a second order Taylor expansion.
\medskip

While we will use an error bound by \cite{steinwart2009optimal}, in related settings very similar convergence results are obtained for instance by \citet{caponnetto2007optimal,fischer2020sobolev,lin2020optimal,mendelson2010regularization,wendland2004scattered,wu2006learning}.
The following result is a special case of Theorem 1 by \citet{steinwart2009optimal} with $s=p=\frac{1}{b}$ and $q=2$.
\begin{lemma}[\citet{steinwart2009optimal}]
	\label{rem:optimalrate}
	There exists a constant $C_l > 0$, such that for $\tau > 0$, it holds
	\[\|\boldsymbol{f}_{\lambda, n} - f_\rho\|_{L^2(\rho_X)}^2 \leq \frac{C_l \tau}{2} \Big(\lambda^{\beta} + \frac{1}{\lambda^\alpha n} + \frac{\lambda^{\beta}}{n} + \frac{1}{n}\Big)\ \leq C_l \tau \Big(\lambda^\beta + \frac{1}{\lambda^\alpha n}\Big)\]
	with probability at least $1-3\exp(-\tau)$ with respect to the $n$-fold product measure $\rho^{\otimes n}$.
	The dominating terms on the right hand side of the inequality  are $\lambda^{\beta}$ and $(\lambda^{\alpha} n)^{-1}$ for $n \rightarrow \infty$, $\lambda \rightarrow 0$. Optimizing for $\lambda$ yields
	$
	\lambda_n := n^{-\frac{1}{\beta + \alpha}}
	$
	and a resulting learning rate of 
	\[
	\|\boldsymbol{f}_{\lambda_n, n} - f_\rho\|_{L^2(\rho_X)}^2 \leq C_l \tau n^{-\frac{\beta}{\beta + \alpha}}
	\]
	with probability at least $1-3\exp(-\tau)$.
\end{lemma}

In the situation of Remark~\ref{rem:sobolev} with a Sobolev space $\mathcal H$ of order $m$ and a $k$-Sobolev functions $f_\rho$, we recover the classical minimax rate $n^{-2k/(2k+D)}$ if $2m>(2k)\vee D$. In particular, the rate deteriorates for large dimensions $D$.

\subsection{Dimensionality Reduction and Transformations of Measures}
\label{subsec:kerneltransformation}
In the previous section, the two parameters $\alpha$ and $\beta$ were introduced which characterize the properties of $\mathcal{H}$ in relation to the data distribution $\rho$ necessary for the rate of convergence, see Lemma \ref{rem:optimalrate}.

When transforming data, it is thus important to study how these parameters change when the data distribution changes. In particular, we focus on two simple yet important cases: First, a linear  $d$-dimensional subspace of $\mathbb{R}^D$ is transformed into a parametrization in $\mathbb{R}^d$, see Lemma \ref{lem:ortho}. And second, independent noise is added to the input data, see Lemma~\ref{lem:noise}.

For the following Lemma \ref{lem:ortho}, the input data is changed with a linear map $A\colon\mathbb{R}^D \rightarrow \mathbb{R}^d$. We identify $A$ by the corresponding matrix $A \in \mathbb{R}^{d \times D}$. We assume that $A$ has orthonormal rows, where one may think of the PCA matrix containing the largest $d$ eigenvectors of the covariance matrix of the data. Formally, the transformation of the input data is the following: If $(X, Y)$ is distributed according to $\rho$, we set $\tilde{\rho}$ as the distribution of $(AX, Y)$, which is a probability measure on $\tilde{\mathcal{Z}} = \tilde{\mathcal{X}} \times \mathcal{Y}$, where $\tilde{\mathcal{X}} \subset \mathbb{R}^d$. We also assume that $A$ is invertible on $\supp(\rho_X)$, meaning that there is an inverse $A_{\rm inv}$ having orthonormal columns and $A_{\rm inv}AX = X$.
The interpretation of this assumption is that the data distribution $\rho_X$ completely lies on a $d$-dimensional plane, see the left and middle images in Figure \ref{fig:lem2and3}.

\begin{lemma}
	\label{lem:ortho}
	Assume that the kernel $K$ is of the form $K(x_1, x_2) = \phi(|x_1-x_2|)$ for some $\phi\colon \mathbb{R} \rightarrow \mathbb{R}$.		
	Then the learning problems for $\rho$ and $\tilde{\rho}$ as defined above are equivalent in the following sense:
	\begin{itemize}
		\item[(i)] $f_{\rho} = f_{\tilde{\rho}} \circ A$
		\item[(ii)] $f_{\lambda, \rho} = f_{\lambda, \tilde{\rho}} \circ A$
		\item[(iii)] The parameters $\alpha$ and $\beta$ from Assumption \ref{ass:alphabeta} can always be chosen equally for the learning problems with $\rho$ and $\tilde{\rho}$.
	\end{itemize}
\end{lemma}
The result comes as no surprise, since both kernel matrices and output variables of the learning problems with $\rho$ and $\tilde{\rho}$ can be transformed into each other by invertible maps. Nevertheless, the result is important in the sense that it perfectly fits the situation of PCA. For this setting, the result formalizes the intuition that it does not matter whether one uses the data as $D$-dimensional points lying on a linear plane or the respective $d$-dimensional parametrization thereof. 

The next result emphasizes the potential benefit of going from a true high-dimensional space to a lower-dimensional one. At least in the extreme case, where all one does is filter out independent noise, this can only improve the eigenvalue decay speed of the kernel integral operator and thus improve on the parameter $\alpha$. Formally, we choose some noise distribution $\kappa$ on $\mathcal{X}$ and denote by $\rho_X * \kappa$ the convolution of $\rho_X$ with $\kappa$. This means, if $X \sim \rho_X$ and $\varepsilon \sim \kappa$ are independent, then $X + \varepsilon \sim \rho_X * \kappa$.


\begin{figure}\label{fig:lem2and3}
	\includegraphics[width=0.325\textwidth]{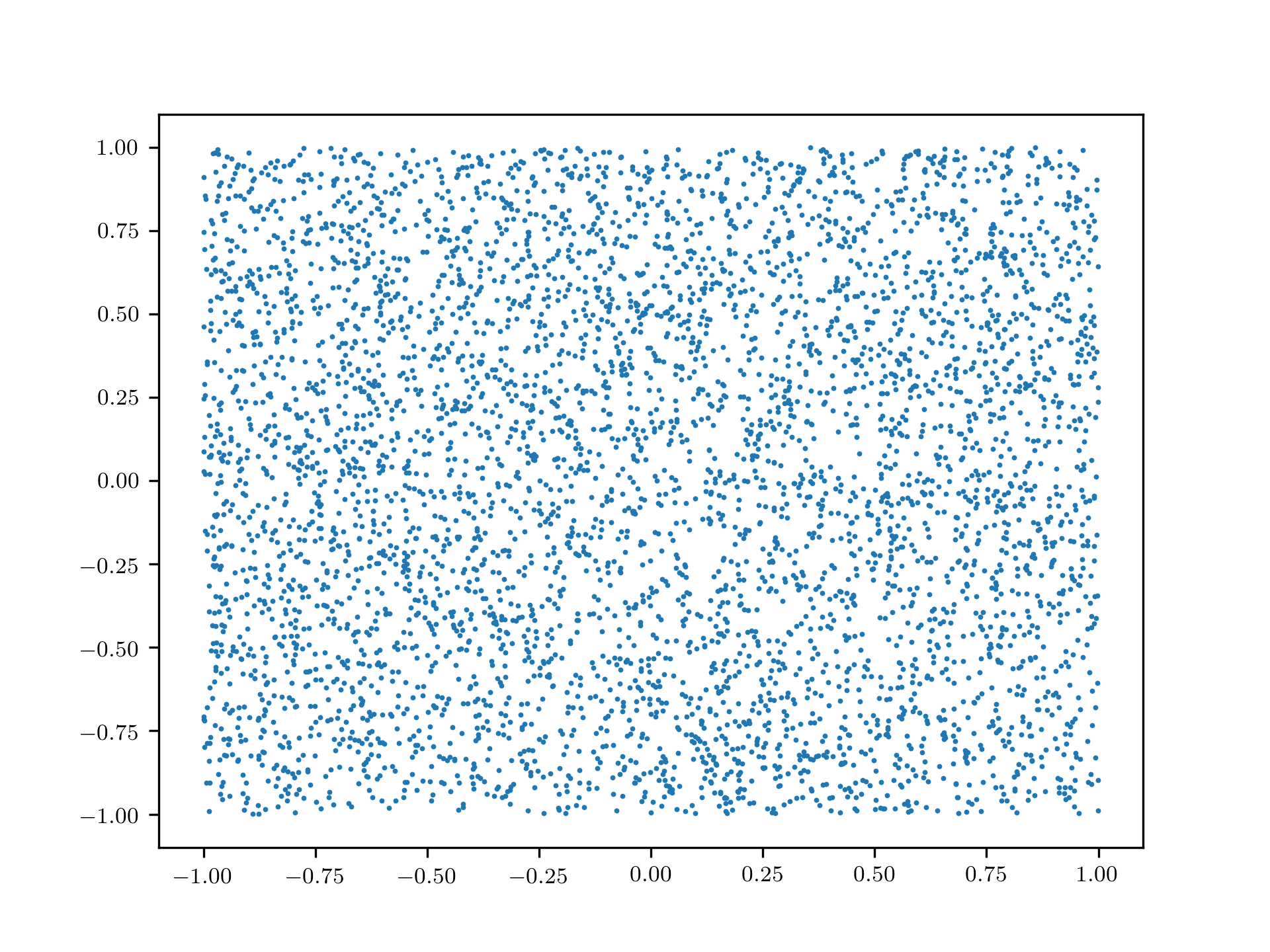}\hspace*{-0.5cm}
	\includegraphics[width=0.375\textwidth]{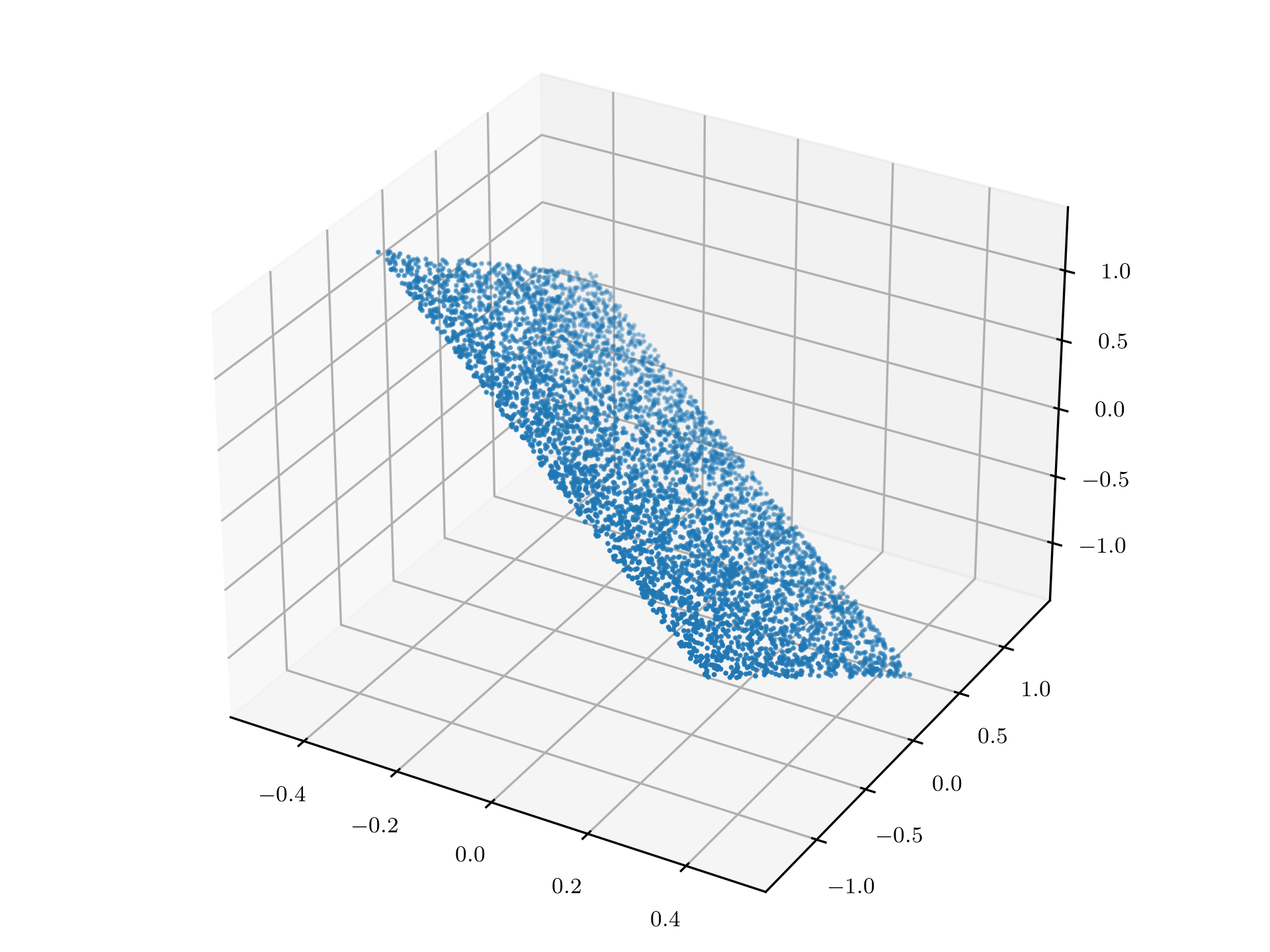}\hspace*{-0.5cm}
	\includegraphics[width=0.375\textwidth]{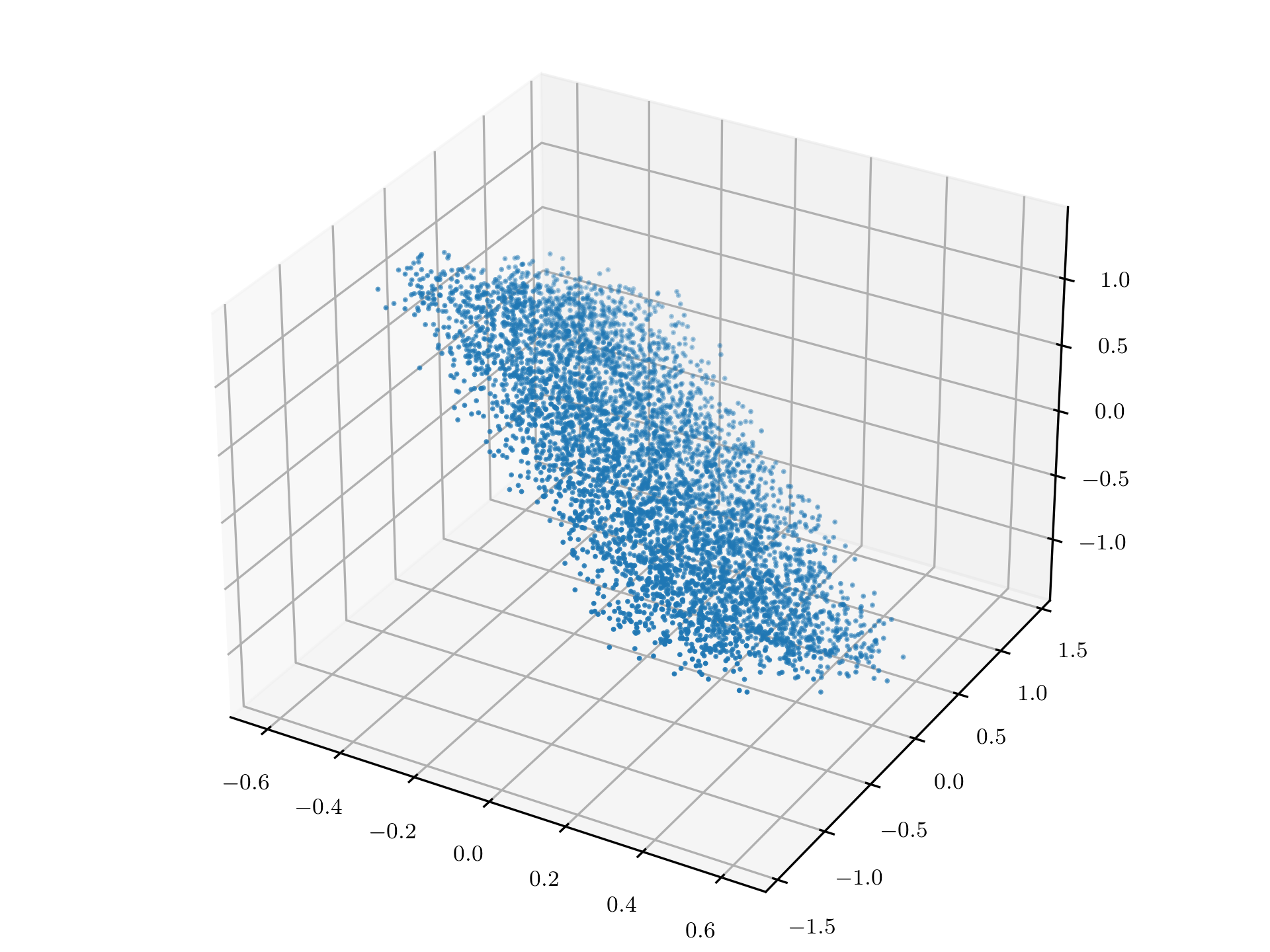}
	\caption{The figure on the left shows a two-dimensional uniform distribution. The figure in the middle illustrates a case where high-dimensional data lies on a lower-dimensional linear subspace and can be transformed with an orthonormal mapping as assumed in Lemma \ref{lem:ortho}. The figure on the right shows a distribution which arises when independent noise is added to data lying on a linear subspace as assumed in Lemma \ref{lem:noise}.}
\end{figure}

\begin{lemma}\label{lem:noise}
	Grant Assumption 2. Assume $\rho_X$ and $\kappa$ have finite first moment. Denote the sequences of decreasing eigenvalues for $L_{K, \rho_X}$, 
	$L_{K, \kappa}$ and $L_{K, \rho_X * \kappa}$ by 
	$(\sigma_n)_{n\in\mathbb{N}}$, $(\gamma_n)_{n\in\mathbb{N}}$ and 
	$(\hat\sigma_n)_{n\in\mathbb{N}}$, respectively. If $\sigma_n \sim n^{-a}$ and $\gamma_n \sim n^{-b}$ for $a, b > 0$. Then, if $\hat{\sigma}_n \lesssim n^{-c}$, we have $c \leq \min\{a, b\}$.
\end{lemma}

Lemma \ref{lem:noise} implies that the decay rate of eigenvalues for the parameter $\alpha$ in Assumption~\ref{ass:alphabeta} is faster for $\rho_X$ compared to $\rho_X * \kappa$, and hence when adding noise, the influence of $\alpha$ goes in the direction of a slower learning rate. In terms of $\alpha$, filtering out noise can thus only improve the learning rate. 
In many situations, the improvement implied by Lemma \ref{lem:noise} is strict as the following remark illustrates.

\begin{remark}
	 Let $X$ be supported on a $d$-dimensional subspace of  $\mathbb{R}^D$ and $\varepsilon$ has $D$-dimensional support (e.g., $\kappa$ is the uniform distribution on $[-\delta, \delta]^D$). For standard choices of kernels (cf.~Remark~\ref{rem:sobolev}, with $D < 2m$, where $m$ is the regularity of the kernel), this leads to $\sigma_n \sim n^{-m/d}$ while $\gamma_n \sim n^{-m/D}$. Thus, by Lemma \ref{lem:noise},
	\[
	\hat{\sigma}_n \gtrsim \max\{\sigma_n, \gamma_n\} = \gamma_n > \sigma_n 
	\]
	for $n$ large enough. Comparing the case with noise ($X + \varepsilon$ with eigenvalues $\hat{\sigma}_n$) to the case without noise ($X$ with eigenvalues $\sigma_n$), the choice of $\alpha$ in Assumption \ref{ass:alphabeta} (i) differs by a gap of at least $\frac{D-d}{m}$.
\end{remark}

While Lemma~\ref{lem:noise} only deals with the parameter $\alpha$, a classical model where the influence of the observation error $\varepsilon \sim \kappa$ 
on the regularity parameter $\beta$ can be determined is the errors-in-variables model \citep[see, e.g.,][]{meister2009}. The following example combines the dimension reduction perspective from above with these errors-in-variables models.
\begin{example}\label{ex:errorsinvars}
 We consider the model
\[
  Z=A^\top X+\eps\in\mathcal X\qquad\text{and}\qquad Y=f(X)+\delta\in\mathcal Y
\]
with $X\in\R^d$, centered observation errors $\eps\in\mathcal X\subseteq\mathbb R^D$, $\delta\in\mathcal Y\subseteq[-M,M]$ where $X,\delta$ and $\eps$ are independent and with a matrix $A\in\R^{d\times D}$ such that $A A^\top=E_d$. An i.i.d. training sample $(Z_i,Y_i)$ distributed as $(Z,Y)$ is observed. Let $\eps$ admit a density $\kappa\colon\R^D\to\R_+$ and $X$ has a density $\rho_{X}\colon\R^d\to\R_+$. In this setting the regression function $f_{\rho}$ is given by
\begin{align*} 
f_{\rho}(z) =  \mathbb E [ Y|Z=z] =  \mathbb E[ 
Y|A^\top X+\varepsilon=z] 
&= \int f(x)\mathbb P(X=dx|A^\top X+\varepsilon=z) \\
&= \frac{\int f(x)\rho_X(x)\kappa(z-A^\top x)dx}{\int \rho_X(x)\kappa(z-A^\top x)dx}.
\end{align*}
A regression with the projected data $AZ=X+A\eps$ as discussed in the context of Lemma~\ref{lem:ortho} leads to
\[
  f_{\tilde\rho}(z^*)=\frac{\int f(x)\rho_X(x)\kappa(A^\top(z^*- x))dx}{\int \rho_X(x)\kappa(A^\top (z^*-x))dx}=\frac{\big((f\cdot \rho_X)\ast \kappa(A^\top\cdot)\big)(z^*)}{\big(\rho_X\ast \kappa(A^\top\cdot)\big)(z^*)}
\]
for any $z^*$ in the range of $A$. In particular, $f_{\tilde\rho}=f_\rho(A^\top\cdot)$ inherits the regularity of $f\rho_X$ and $\kappa$ thanks to the convolution structure. In contrast, $f_\rho$ only achieves the regularity of $\kappa$ in directions which are not aligned with $A^\top$ such that $f_{\tilde\rho}$ might be much more regular than $f_\rho$, see Figure~\ref{fig:eivm}. This exemplifies a case where the regularity parameter $\beta$ from Assumption \ref{ass:alphabeta} may improve through dimensionality reduction.
\end{example}
\begin{figure}\label{fig:errorsinvar}
	\includegraphics[height=0.27\textheight, width=0.46\textwidth]{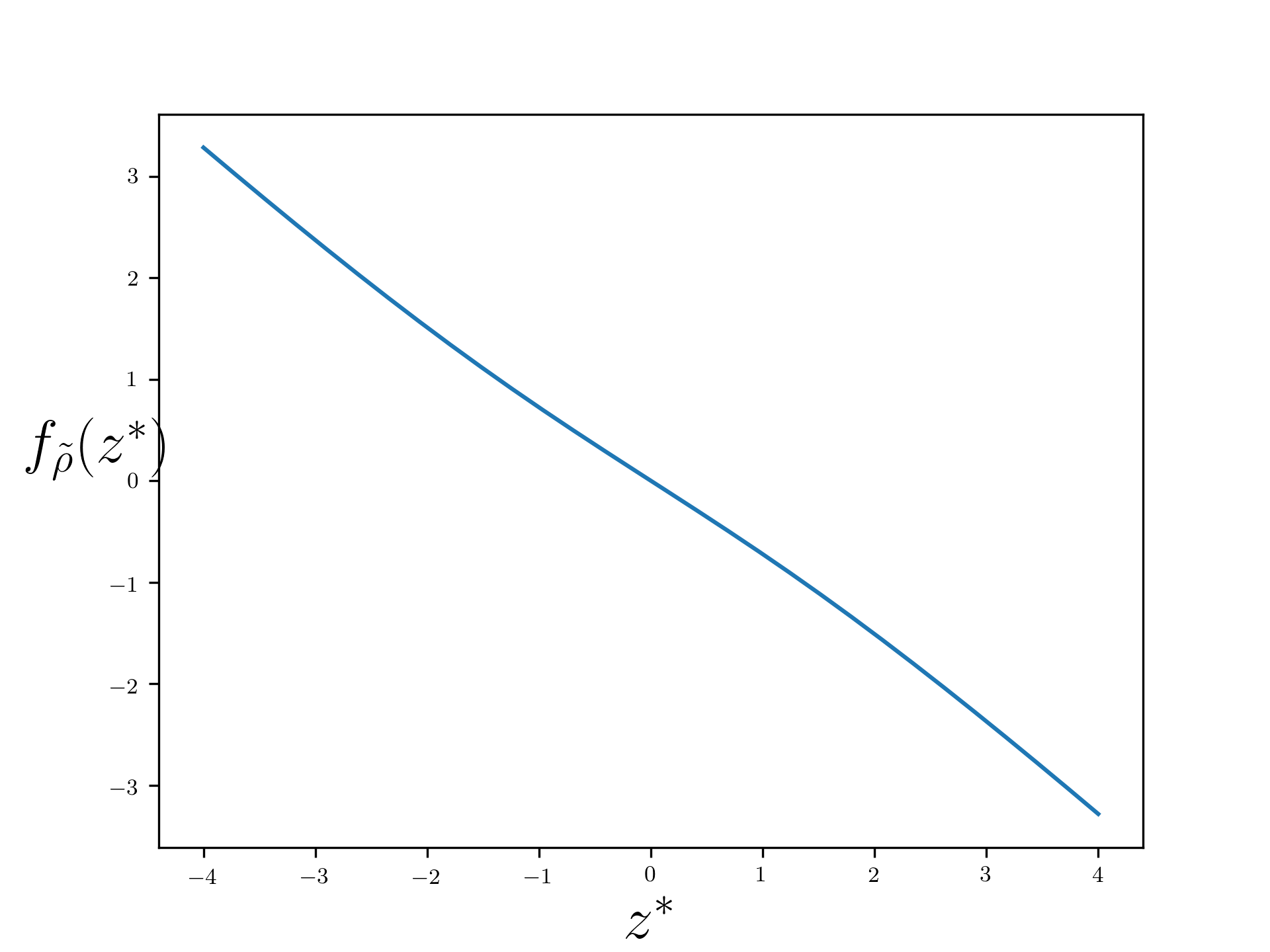}
	\includegraphics[height=0.30\textheight, width=0.53\textwidth]{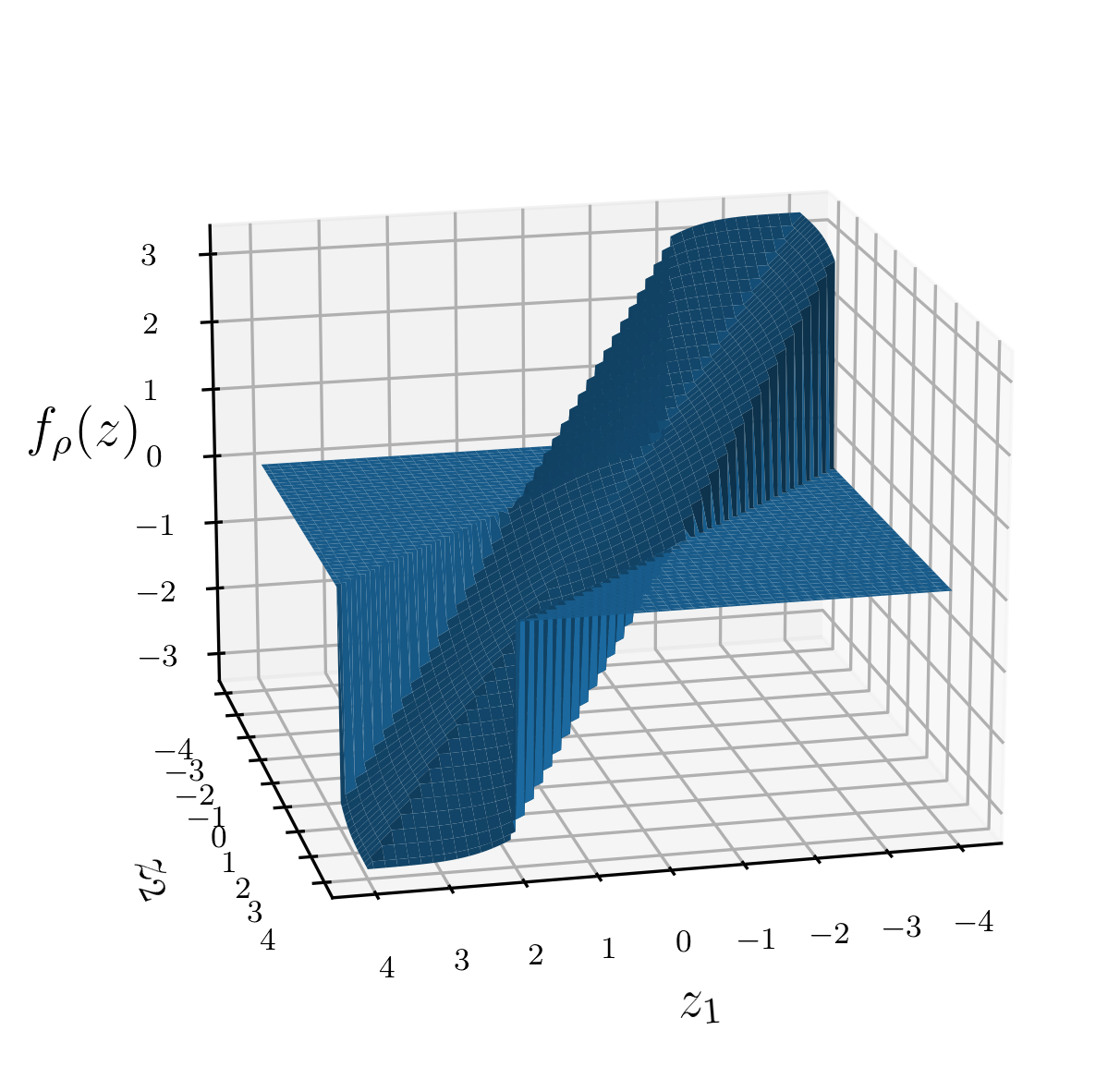}
	\caption{An illustration of the errors-in-variables model from Example \ref{ex:errorsinvars} where the low-dimensional regression function (left) is much smoother than the high-dimensional regression function (right). The plotted functions arise in the errors-in-variables model by setting $X \sim \mathcal{N}(0, 1)$, $A = (1, 1)$, $\varepsilon \sim \mathcal{U}([-1, 1]^2)$ and $f(x) = x$.}\label{fig:eivm}
\end{figure}


\subsection{Stability Result}
\label{subsec:kernelstability}

This section presents a general stability result for kernel regression with respect to the Wasserstein distance. In Section \ref{sec:secmain}, this result will be applied to the case where the error in the data arises from a dimension reduction preprocessing step. We first state the general result in \mbox{Theorem \ref{thm:stab}}, and discuss some of its aspects subsequently.

\begin{definition}
	For two probability measures $\rho_1$ and $\rho_2$ on $\mathcal{Z} = \mathcal{X} \times \mathcal{Y}$ we define the Wasserstein distance $W_1(\rho_1, \rho_2)$ between $\rho_1$ and $\rho_2$ by
	\[
	W_1(\rho_1, \rho_2) := \inf_{\pi \in \Pi(\rho_1, \rho_2)} \int \big(|x_1-x_2| + |y_1 - y_2|\big) \,\pi(d(x_1, y_1), d(x_2, y_2))
	\]
	where $\Pi(\rho_1, \rho_2)$ is the set of all probability measures $\pi$ on $\mathcal{Z}^2$ with first marginal distribution $\rho_1$ and second marginal distribution $\rho_2$.
\end{definition}

\begin{theorem}
	\label{thm:stab}
	Grant Assumption \ref{ass:kernel}. If three distributions $\rho_1, \rho_2, \rho_3$ on $\mathcal{Z}$ satisfy the relation
	\begin{equation}
	\label{eq:assstab}
	W_1(\rho_1,\rho_3)\le\frac\lambda{4L}
	\end{equation}
	for $\lambda>0$, then it holds
	\[
	\|f_{\lambda, \rho_1} - f_{\lambda, \rho_2}\|_{L^2(\rho_3)} \leq \big(\max\{1, L\,M\}+L\|f_{\lambda,\rho_2}\|_\infty + \|f_{\lambda,\rho_2}\|_{\mathrm{Lip}}\big) \frac{W_1(\rho_1, \rho_2)}{\sqrt{\lambda}}.
	\]
	The conclusion remains true if we replace Condition~\eqref{eq:assstab} by 
		\begin{equation}\label{rem:stab}
		\|(T_{\rho_3} - T_{\rho_1})(T_{\rho_3}+\lambda I)^{-1}\|_{\rm op} \leq \frac{1}{2}.
		\end{equation}
\end{theorem}


In Theorem \ref{thm:stab} the assumption on the relation between $\rho_1$ and $\rho_3$ is obviously satisfied for $\rho_1 = \rho_3$. Further, as will later be used, it is satisfied when $\rho_1$ is an empirical measure of $\rho_3$ of sufficiently large sample size, relative to $\lambda$. 
The assumption of Lipschitz continuity of the kernel is made so that it 
fits the standard definition of the Wasserstein distance. One may, for 
instance, weaken the assumption of Lipschitz continuity to H\"{o}lder
continuity with exponent $a \in (0, 1)$ (that is, weaken the growth condition in 
Assumption \ref{ass:kernel} to $\phi(0) - \phi(r) \leq r^{2 a}$) 
while simultaneously adjusting the cost function of the Wasserstein 
distance to $|x_1 - x_2|^a + |y_1 - y_2|$.

\begin{remark}
	\label{rem:constants}
	Due to Assumption~\ref{ass:kernel} the norms $\|f_{\lambda,\rho_2}\|_{\infty}$ and $\|f_{\lambda,\rho_2}\|_{\mathrm{Lip}}$ on the right-hand side of the inequality in Theorem \ref{thm:stab} can be bounded by 
	$\|f_{\lambda, \rho_2}\|_{\mathcal{H}}$, which can be controlled in 
	various ways. Notably, it always holds $\|f_{\lambda, 
		\rho_2}\|_{\mathcal{H}} \leq \frac{M}{\sqrt{\lambda}}$ and more 
	generally, stronger bounds on the $\mathcal{H}$-norm are for 
	instance given in \citet[Lemma 5.2]{wu2006learning} and indirectly 
	by \citet[Theorem 3]{cucker2002mathematical}. In particular, if 
	$f_{\rho_2} \in \mathcal{H}$, then $\|f_{\lambda, 
		\rho_2}\|_{\mathcal{H}}$ is uniformly bounded for all $\lambda$, see also Lemma~\ref{lem:normbound} below.
\end{remark}

\begin{example}
	We exemplify Theorem \ref{thm:overallerror} in the setting of Lemma \ref{lem:noise}. Say $(X, Y) \sim \rho$ and $\varepsilon \sim \kappa$, where $(X, Y)$ and $\varepsilon$ are independent. Assume 
	\[
	\int |x| \,\kappa(dx) = \delta.
	\]
	Define $(X + \varepsilon, Y) \sim \tilde{\rho}$. Then it is straightforward to see that $W_1(\rho, \tilde{\rho}) = \delta$ and thus for some constant $B_\lambda$ only depending on $\lambda$ but independent of $\delta$, Theorem \ref{thm:stab} yields that
	\[
	\|f_{\lambda, \rho} - f_{\lambda, \tilde{\rho}}\|_{L^2(\rho)} \leq B_\lambda \delta.
	\]
	As mentioned in Remark \ref{rem:constants}, both the Lipschitz 
	constant and the $L^\infty$ bound on $f_{\lambda, \tilde{\rho}}$ 
	which are contained in the constant $B_\lambda$ can be 
	controlled independently of $\tilde{\rho}$ by 
	$\|f_{\lambda, \tilde{\rho}}\|_{\infty}+\|f_{\lambda, \tilde{\rho}}\|_{\rm Lip} \leq 
	(1+L) \|f_{\lambda, \tilde{\rho}}\|_{\mathcal{H}} \leq 
	\frac{M}{\sqrt{\lambda}}$. 
\end{example}

We discuss the order of magnitude of the relation implied by Theorem \ref{thm:stab} for the difference $\|f_{\lambda, \rho_1} - f_{\lambda, \rho_2}\|_{L^2(\rho_3)}$ in terms of both $W_1(\rho_1, \rho_2)$ and $\lambda$ in the following example. 
\begin{example}
 Let $\rho_{1}=\rho_{3}=\frac{1}{2}\delta_{(0,0)}+\frac{1}{2}\delta_{(b,1)}$ and $\rho_{2}=\frac{1}{2}\delta_{(a,0)}+\frac{1}{2}\delta_{(b+a,1)}$ for $a,b\in(0,1)$ and $\lambda\in(0,1]$. In particular, we have $W_1(\rho_1,\rho_2)=a$. By the representer theorem we have 
\begin{align*}
f_{\lambda,\rho_{1}}(x)&=\hat{w}_{1}K\big(0,x\big)+\hat{w}_{2}K\big(b,x\big)\qquad\text{and}\\
f_{\lambda,\rho_{2}}(x)&=\hat{w}_{1}K\big(a,x\big)+\hat{w}_{2}K\big(b+a,x\big)
 \end{align*}
where $\hat{w}=(\hat{w}_{1},\hat{w}_{2})^{\top}$ is given by
\[\hat{w}	=\argmin_{w\in\R^2}\frac{1}{2}\left|\big(0,1\big)^{\top}-\mathbf{K}w\right|^{2}+\lambda w^{\top}\mathbf{K}w,\qquad\mathbf{K}=\begin{pmatrix}1 & K(0,b)\\
K(0,b) & 1
\end{pmatrix}.
\]
We choose $b$ such that $K(0,b)=1-\lambda$. Due to Assumption~\ref{ass:kernel}, this implies that $b$ is of order $\sqrt \lambda$. Basic linear algebra yields $\|f_{\lambda,\rho_{1}}-f_{\lambda,\rho_{2}}\|_{L^{2}(\rho_{1})}\geq c \frac{ab}\lambda$ and $\|f_{\lambda,\rho_{2}}\|_{\mathcal{H}}\le C \lambda^{-1/2}$ for some constants $C,c>0$ which do not depend on $\lambda$ and $a$. We conclude from Theorem \ref{thm:stab} that
\[
c\frac{a}{\sqrt{\lambda}}\le\|f_{\lambda,\rho_{1}}-f_{\lambda,\rho_{2}}\|_{L^{2}(\rho_{3})} \le C'\frac{1+\|f_{\lambda,\rho_{2}}\|_{\infty}+\|f_{\lambda,\rho_{2}}\|_{\mathrm{Lip}}}{\sqrt{\lambda}}W(\rho_{1},\rho_{2})\le C''\frac{a}{\lambda}
\]
with $C',C''>0$. We recover the linear dependence on $W_1(\rho_{1},\rho_{2})$. Although the order in $\lambda$ is not sharp in this simple example, it already reveals that the distance $\|f_{\lambda,\rho_{1}}-f_{\lambda,\rho_{2}}\|_{L^{2}(\rho_{3})}$ might indeed explode for $\lambda\to0$.
\end{example}

\section{Dimensionality Reduction}
\label{sec:secmain}
This section studies the combined approach of dimensionality reduction and function estimation. First, we derive an error estimate for the two-step procedure allowing for a general dimension reduction method. Subsequently, this result is detailed out for principal compoment analysis.

\subsection{General Dimensionality Reduction Estimates}
\label{subsec:generaldimred}

As before, the training data is given by pairs
\[(X_1, Y_1), (X_2, Y_2), \dots, (X_n, Y_n) \in \mathcal{Z}\]
which are independent and identically distributed according to $\rho$. 
We first consider the general situation where we have some optimal dimensionality reduction map
\[
P\colon \mathbb{R}^D \rightarrow \mathbb{R}^d
\]
with $d \leq D$ and some estimator $\hat{P}_n$. Hereby, $P$ is a dimensionality reduction for our explanatory data assuming full knowledge of the distribution of $X_1 \sim \rho_X$, while $\hat{P}_n$ is an estimator depending on $X_1, \dots, X_n$. In this subsection, we will not specify $P$ and $\hat{P}_n$ any further, and derive general error estimates depending on the difference $\|\hat{P}_n - P\|_{\rm op}$. In particular, the map $P$ does not need be linear in this subsection. More specific estimates for the case of PCA will be derived in Sections \ref{sec:PCA} and \ref{sec:PCAkernel}. For further examples of dimensionality reduction methods we refer to \cite{lee2007nonlinear}, noting in particular that kernel PCA methods may yield suitable estimates as well, cf.~\cite{scholkopf1997kernel} or \cite{reiss2020nonasymptotic}.

We define $X_i^* := P (X_i)$ and the estimated counterpart as $\hat{X}_i := \hat{P}_n (X_i)$. 
We are looking for the optimal regression function given the dimension-reduced input data. Denoting the underlying distribution of the reduced data $(X_1^*,Y_1)$ by $\tilde{\rho}$, the best predictor is given by
\[
f_{\tilde{\rho}}(x) = \int y \,\tilde{\rho}(dy | x).
\]
Since we lack knowledge of both $\rho$ and $P$, we do not know $\tilde{\rho}$ either. To estimate $f_{\tilde{\rho}}$ with our finite data $(\hat{X}_1, Y_1), \dots, (\hat{X}_n, Y_n)$, we define the regularized least square kernel fit
\begin{align*}
\hat{f}_{\lambda, n} := \argmin_{f \in \mathcal{H}} \frac{1}{n} \sum_{i=1}^n \big(f(\hat{X}_i) - Y_i\big)^2 + \lambda \|f\|_{\mathcal{H}}^2.
\end{align*}
For our error analysis we will later also require
\begin{align*}
f_{\lambda, n}^* &:= \argmin_{f \in \mathcal{H}} \frac{1}{n} \sum_{i=1}^n \big(f(X^*_i) - Y_i\big)^2 + \lambda \|f\|_{\mathcal{H}}^2.
\end{align*}
As before, $\hat{\boldsymbol{f}}_{\lambda, n}$ and $\boldsymbol{f}^*_{\lambda, n}$ are the clipped versions.

The mappings $\hat{P}_n$ and $\hat{f}_{\lambda, n}$ can be calculated given the sample data, while $P$ and $f_{\tilde{\rho}}$ are the respective best possible fits that can result from the employed procedure.
Of main interest in our analysis is thus the estimation error 
\begin{equation}
\label{eq:mainerror}
\| \hat{f}_{\lambda, n} \circ \hat{P}_n - f_{\tilde{\rho}} \circ P \|_{L^2(\rho_X)}.
\end{equation}
\MT{In addition there is} the structural error arising from the dimensionality reduction, i.e., $\|f_{\tilde{\rho}} \circ P - f_\rho\|_{L^2(\rho_X)}$. This error of course highly depends on the structure of the data $X$ and the suitability of the applied dimensionality reduction $P$. A bound for the structural error in case of PCA depending on the linear structure of $X$ will be given in Proposition~\ref{prop:procedureerror}.

First, we give the specification of Theorem \ref{thm:stab} to the case 
relevant for bounding the error in Equation \eqref{eq:mainerror}. To this end, Proposition \ref{prop:pcakernelstability} bounds the error between the kernel regression problems using differing input data determined by either $P$ or $\hat{P}_n$. 
We recall that $\mathcal{Y} \subseteq [-M, M]$ and the constant $L$ coming from Assumption \ref{ass:kernel}. 

\begin{proposition}
	\label{prop:pcakernelstability}
	Let $\eta \in (0, 1)$ and grant Assumption \ref{ass:kernel}.
	For $\mathcal{N}(\lambda) := 
\Tr ((T_{\tilde{\rho}} + \lambda)^{-1} T_{\tilde{\rho}})$, assume that
	\begin{equation}
	\label{eq:assnlargeenough}
	\lambda \leq \|T_{\tilde{\rho}}\|_{\rm op} ~~ \text{ and } ~~ n \geq 64 \frac{\log^2(6 \eta^{-1}) \mathcal{N}(\lambda)}{\lambda}
	\end{equation}
	and set $S_n := \frac{1}{n} \sum_{i=1}^n |X_i|$. Then 
		\[
		\|\hat{f}_{\lambda, n} \circ P - f_{\lambda, n}^* \circ P\|_{L^2(\rho_X)} \leq \frac {S_n  \big(1+LM+\|\hat f_{\lambda,n}\|_{\mathrm{Lip}} + \|\hat{f}_{\lambda, n}\|_\infty \big)}{\sqrt{\lambda}} \, \big\|\hat{P}_n - P\big\|_{\rm op}
		\]
		holds with probability at least $1-\eta/3$ with respect to $\rho^{\otimes n}$.
\end{proposition}
In the above result, the term $S_n$ is of course either bounded (if $X_1$ is bounded) or at least bounded with high-probability by the law of large numbers, since we will always assume that $X_1$ has a finite second moment. 
Further, while $\|\hat f_{\lambda,n}\|_{\mathrm{Lip}} + \|\hat{f}_{\lambda, n}\|_\infty$ can be unbounded in general, in various situations one may reasonably assume that this term is (uniformly across $n$) bounded with high probability. One case may be if $\hat{f}_{\lambda, n}$ converges in a sufficiently strong sense to a bounded and Lipschitz continuous function. Another case may be if all estimators are uniformly bounded with respect to $\|\cdot\|_{\mathcal H}$ (cf.~Remark \ref{rem:constants}).
Finally, it is worth mentioning that the assumption in \eqref{eq:assnlargeenough} may be restrictive in general, although in the setting of \cite{caponnetto2007optimal} the assumption is shown to be satisfied asymptotically for a suitable choice of $\lambda$ leading to optimal rates \citep[cf.~the related discussion in][Chapter~13]{wendland2004scattered}.

Below, we state the main result of this section, which shows how to reduce the error of the joint procedure (estimation of both dimensionality reduction and kernel regression) to each individual procedure, where only the results for the kernel regression on the low-dimensional explanatory data (given by the push-forward distribution $\rho_X \circ P^{-1}$) is needed.
\begin{theorem}\label{thm:generaldimred}
	Let $\lambda\in(0,1]$ and $\eta \in (0, 1)$ such that \eqref{eq:assnlargeenough} holds. Define $S^2=\mathbb E[|X_1|^2]$ and assume $S < \infty$. Then
	\begin{align*}
	&\frac{1}{3}\|\hat{\boldsymbol{f}}_{\lambda, n} \circ \hat{P}_n - f_{\tilde{\rho}}\circ P \|^2_{L^2(\rho_X)} \\
	&\leq \frac {(S_n^2+S^2)  \big(1+LM+\|\hat f_{\lambda,n}\|_{\mathrm{Lip}} + \|\hat{f}_{\lambda, n}\|_\infty \big)^2}{\lambda} \, \big\|\hat{P}_n - P\big\|^2_{\rm op} + \|\boldsymbol{f}_{\lambda, n}^* - f_{\tilde{\rho}}\|^2_{L^2(\rho_{X}\circ P^{-1})}
	\end{align*}
	holds with probability at least $1-\eta/3$ with respect to $\rho^{\otimes n}$.
\end{theorem}
To prove this theorem, we decompose
\begin{eqnarray}
\frac{1}{3}\|\hat{\boldsymbol{f}}_{\lambda, n} \circ \hat{P}_n - f_{\tilde{\rho}} \circ P\|^2_{L^2(\rho_X)}
&\leq&\|(\hat{f}_{\lambda, n} \circ \hat{P}_n - \hat{f}_{\lambda, n} \circ P)\|^2_{L^2(\rho_X)} \notag\\
&&\mbox{}+ \|(\hat{f}_{\lambda, n} \circ P - f_{\lambda, n}^* \circ P)\|^2_{L^2(\rho_X)} \label{eq:decomp}\\
&&\mbox{}+ \|(\boldsymbol{f}_{\lambda, n}^* \circ P - f_{\tilde{\rho}} \circ P)\|^2_{L^2(\rho_X)},\notag
\end{eqnarray}%
where the first term can be estimated with the Lipschitz constant of $\hat{f}_{\lambda, n}$ combined with the reconstruction error $\|\hat P_n-P\|_{\rm op}$ and the second term is bounded by  Proposition~\ref{prop:pcakernelstability}.

Theorem~\ref{thm:generaldimred} implies that as soon as we can control $\|\hat{P}_n - P\|_{\rm op}$, we can reduce the estimation error $\|\hat{\boldsymbol{f}}_{\lambda, n} \circ \hat{P}_n - f_{\tilde{\rho}}\circ P \|_{L^2(\rho_X)}$ of the combined learning problem to the estimation error which would result if we just had to estimate $f_{\tilde{\rho}}$ while assuming perfect knowledge of $P$, i.e., to $\|\boldsymbol{f}_{\lambda, n}^* - f_{\tilde{\rho}}\|_{L^2(\rho_{X}\circ P^{-1})}$.

\subsection{Principal Component Analysis}
\label{sec:PCA}
Next, we state results on PCA, which are mainly taken from 
\cite{reiss2020nonasymptotic}. For a more general treatment of PCA in the context of dimensionality reduction we refer for instance to \citet{jolliffe2002principal, lee2007nonlinear}. The results by \cite{reiss2020nonasymptotic} will later be used in our 
estimates for the combined study of PCA with regularized kernel regression. 
The goal is to reduce the dimension of the input data from dimension $D$ to 
$d < D$. The reduced data will still be regarded as points in 
$\mathbb{R}^D$, albeit Lemma \ref{lem:ortho} shows that it does not matter 
whether one uses the points as a linear subset in $\mathbb{R}^D$ or a 
$d$-dimensional representation thereof.

For the PCA procedure, we are in the same setting as introduced in \mbox{Section \ref{subsec:kernelsetting}}. Let
\begin{equation}
\label{eq:orthoproj}
\mathcal{P}_d := \{ P \in \mathbb{R}^{D \times D} : P \text{ is orthogonal projection of rank } d\}. 
\end{equation}
This means $P \in \mathcal{P}_d$ can be written as $P = A^\top A$ for $A \in \mathbb{R}^{d \times D}$ having orthonormal rows. Define
\begin{equation*}
P := \argmin_{\tilde{P} \in \mathcal{P}_d} \int |x - \tilde{P}x|^2 \,\rho_X(dx)\qquad\text{and}\qquad
\hat{P}_n := \argmin_{\tilde{P} \in \mathcal{P}_d} \frac{1}{n} \sum_{i=1}^n |X_i - \tilde{P} X_i|^2.
\end{equation*}

The following assumption will be used to apply the results on principal component analysis. Note that part (i) implies that $X$ has finite second moment.
\begin{assumption}
	\label{ass:PCA}
	Let $X \sim \rho_X$.
	\begin{itemize}
		\item[(i)] $X$ is sub-Gaussian, i.e., $\sup_{u\in\R^D:\mathbb E[\langle X,u\rangle^2]\le 1}\sup_{k\ge 1}k^{-1/2}\mathbb E[|\langle X,u\rangle|^k]^{1/k}<\infty$.
		\item[(ii)] $\sigma^X_d - \sigma^X_{d+1} > 0$, where $\sigma^X_1, \sigma^X_2, \dots, \sigma^X_D$ is the sequence of decreasing eigenvalues of the covariance matrix of $X$.
		\item[(iii)] $X$ is centered: $\mathbb{E}[X] = 0$.
	\end{itemize}
\end{assumption}

The following is a corollary of \citet[Proposition 2]{reiss2020nonasymptotic} and the main reconstruction result for PCA which will be used in Section \ref{sec:PCAkernel}.
\begin{lemma}[\citet{reiss2020nonasymptotic}]
	\label{lem:pcaerror}
	Grant Assumption \ref{ass:PCA}.
	There exists a constant $C_{pca}$ such that for all $n \in \mathbb{N}$, $\eta \in (0, 1)$, with probability $(1-\eta)$, it holds
	\[
	\|\hat{P}_n - P\|_{\rm op} \leq \|\hat{P}_n - P\|_2 \leq \frac{\Gamma}{3\eta\sqrt n}\qquad\text{with}\qquad \Gamma:=\frac{3C_{pca}}{(\sigma^X_{d+1} - \sigma^X_d)}.
	\]
\end{lemma}

The PCA projections we work with defined in Equation \eqref{eq:orthoproj} are regarded as mappings from $\mathbb{R}^D$ onto itself. In terms of matrices, this means the projection mapping $P = A^\top A$ is used to get the reduced data $P X \in \mathbb{R}^D$ lying on a $d$-dimensional plane. An alternative route would be to work with $A \in \mathbb{R}^{d \times D}$ directly instead and obtain reduced data $A X \in \mathbb{R}^d$. Similarly to Lemma \ref{lem:pcaerror} one can obtain error estimates in this setting using the $\sin\theta$-Theorem \citep[see, e.g.,][]{yu2015useful} combined with optimal rates on covariance matrix estimation \citep[see, e.g.,][]{cai2010optimal}.

\subsection{Estimation Error for PCA and Kernel Regression}
\label{sec:PCAkernel}	

We can now bound the statistical error from \eqref{eq:mainerror} for PCA and kernel regression. Recall the constants $C_l$ and $\Gamma$ from Lemma \ref{rem:optimalrate} and Lemma~\ref{lem:pcaerror}, respectively. The parameter $\beta$ is as in Assumption \ref{ass:alphabeta} for the measure $\tilde{\rho}$ instead of $\rho$. As before, $S_n := \frac{1}{n} \sum_{i=1}^n |X_i|$ and $S^2=\mathbb E[|X_1|^2]$, and we assume $S < \infty$ throughout.

\begin{theorem}
	\label{thm:overallerror}
	Grant Assumptions \ref{ass:kernel} and \ref{ass:PCA} and let $\tilde\rho$ satisfy \eqref{eq:assbeta} for $\beta \in(0, 1]$. Let $\lambda\in(0,1]$ and $\eta \in (0, 1)$ such that \eqref{eq:assnlargeenough} holds. Then there is a constant $\bar C>0$ depending on $M,L$ and $\Gamma$ such that
	\begin{align}
	\begin{split} 
	\label{eq:overallerror}
	\|\hat{\boldsymbol{f}}_{\lambda, n} \circ \hat{P}_n - f_{\tilde{\rho}}\circ P \|^2_{L^2(\rho_X)} 
	&\leq \bar C\frac{(S^2+S_n^2) (\|\hat f_{\lambda,n}\|_{\mathrm{Lip}}^2 + \|\hat{f}_{\lambda, n}\|_\infty^2 + 1)}{\eta^2 n\lambda } \\
	&\qquad+ 3 C_l \log\big(\frac 9\eta\big) \Big(\lambda^{\beta} + \frac{1}{\lambda n}\Big).
	\end{split}
	\end{align}
	with probability at least $1-\eta$ according to the $n$-fold product measure $\rho^{\otimes n}$. 
\end{theorem}
To show this result, we apply Theorem~\ref{thm:generaldimred} in combination with the PCA error bound from Lemma~\ref{lem:pcaerror} and the kernel regression error bound on the low-dimensional space from Lemma~\ref{rem:optimalrate}.

In the following, we specify the above Theorem \ref{thm:overallerror} by optimizing for $\lambda$ under the simplifying assumption that $\|\hat{f}_{\lambda, n}\|_{\rm Lip} + \|\hat{f}_{\lambda, n}\|_\infty$ is bounded. Since both norms are bounded by $\|\hat f_{\lambda,n}\|_{\mathcal H}$, the following lemma provides a high probability bound in the regular case $\beta=1$:
\begin{lemma}\label{lem:normbound}
If $f_{\tilde{\rho}}\in\mathcal{H}$
and $\lambda=n^{-1/2}$, then there is a constant $B>0$ such that
\[
\|\hat f_{\lambda,n}\|_{\mathcal{H}}^{2}\le B\eta^{-3/2}+4\|f_{\tilde{\rho}}\|_{\mathcal{H}}^{2}
\]
is satisfied with probability of at least $1-5\eta$.
\end{lemma}
On the event that $\|\hat{f}_{\lambda, n}\|_{\rm Lip} + \|\hat{f}_{\lambda, n}\|_\infty$ is bounded, the dominating terms in the bound of Theorem \ref{thm:overallerror} for $n\rightarrow \infty$ and $\lambda \rightarrow 0$ on the right hand side of Equation \eqref{eq:overallerror} are $\frac{1}{n \lambda}$ and $\lambda^\beta$. Note that the resulting optimal $\lambda$ satisfies \eqref{eq:assnlargeenough} as soon as $\beta>\alpha$.
\begin{corollary}
	\label{cor:errorsummary}
	We are in the setting of Theorem \ref{thm:overallerror}.
	Optimizing the right hand side of Equation \eqref{eq:overallerror} for $\lambda$ yields 
	\[
	\lambda = n^{-\frac{1}{\beta + 1}}.
	\]
	Assuming \eqref{eq:assnlargeenough}, there are high-probability events $\mathcal A_n\subseteq\Omega, n\in\mathbb N,$ with $\mathbb P(\mathcal A_n) \geq 1 - \eta_n$ such that we have an overall learning rate of the form
	\begin{align}
	\begin{aligned}
	\label{eq:summaryrate}
	\|\hat{\boldsymbol{f}}_{\lambda, n} \circ \hat{P}_n - f_{\tilde{\rho}}\circ P \|^2_{L^2(\rho_X)} &\leq \frac{C}{\eta_n^2}\, n^{-\frac{\beta}{\beta + 1}} \\ &\text{on}~ \mathcal A_n\cap\{\|\hat{f}_{\lambda, n}\|_{\rm Lip} + \|\hat{f}_{\lambda, n}\|_\infty\le \mathcal L\}
	\end{aligned}
	\end{align}
	with some constant $C>0$ (which depends on $\mathcal{L}$).
\end{corollary}

Instead of the procedure introduced in this section (first PCA, then kernel regression on $d$-dimensional subspace), consider the alternative method to fit $f_\rho$ directly using kernel regression on $\mathbb{R}^D$. Let us denote the clipped estimated kernel functions by $\boldsymbol{g}_{\lambda, n}$ and $\overline{\alpha}, \overline{\beta}$ be the parameters for Assumption \ref{ass:alphabeta} for $\rho$. By Lemma \ref{rem:optimalrate}, the resulting learning rate is
\[
\|f_\rho - \boldsymbol{g}_{\lambda, n}\|^2_{L^2(\rho_X)} \leq C n^{-\frac{\overline{\beta}}{\overline\beta + \overline{\alpha}}}.
\]
Compared with the rate in \eqref{eq:summaryrate}, we find that it is not immediate which procedure converges faster. Indeed, only if \[\frac{\overline{\beta}}{\overline{\alpha}} < \beta,\] then the two-step approach converges faster than a direct high-dimensional estimation. Note that it is natural to assume $\beta\ge\bar\beta$, i.e. the lower dimensional regression function is as least as regular as the high-dimensional regression function. In view of Example~\ref{ex:errorsinvars} we may indeed gain in the regularity, i.e. $\beta>\bar\beta$.

In Corollary \ref{cor:errorsummary}, the influence of $\alpha$ from Assumption \ref{ass:alphabeta} is lost in the overall rate. More precisely, the rate behaves as if $\alpha$ takes the (worst possible) value of $\alpha=1$. In a sense, this means that the rate is no longer adaptive with respect to the complexity of $\rho_X$ measured by the RKHS $\mathcal{H}$. This makes intuitive sense, since now the complexity of $\rho_X$ also influences the estimation procedure of the PCA map, which is independent of $\mathcal{H}$. 

On the other hand, in the setting of Remark~\ref{rem:sobolev} but for 
the low dimensional space, we can adapt the regularity $m$ of the kernel to 
the dimension $d$ which has two advantages: First, we can allow for much 
smaller regularities since $d$ might be considerably smaller than $D$ and 
thus the numerical stability of the kernel method can be improved \citep[cf.][Section 12]{wendland2004scattered}. Second, 
the resulting $\alpha=\frac{d}{2m}$ can be chosen close to one and the rate 
in 
Corollary~\ref{cor:errorsummary} is only slightly worse than the best 
possible rate $n^{-2k/(2k+d)}$ for $k$-Sobolev regular regression functions 
$f_{\tilde \rho}$.

Moreover, this comparison certainly motivates that dimensionality reduction is particularly suitable to apply if one does not expect small values of $\overline{\alpha}$ anyways. Recall that $\overline{\alpha}$ is usually small if the kernel $K$ is very smooth and the dimension is not too large. Especially for large dimensions, one could expect high values of $\overline{\alpha}$, and thus not much is lost when applying dimensionality reduction, which basically leads to $\alpha = 1$. This intuition is consistent with the numerical results in Section \ref{sec:examples}, where Figure \ref{fig:case1and2} shows that utilizing PCA for dimensionality reduction leads to notable improvements for $C^0$ or $C^2$ kernels, while this is not (or less so) the case for the Gaussian kernel, which is $C^\infty$.

\medskip
The two-step procedure has another interesting advantage. Consider the case where more $X$-sample points are given compared to $(X, Y)$ sample pairs, which is a situation frequently observed in semi-supervised learning settings. Then, the error bounds for the PCA map can use all $X$ samples, even though the function estimation can only use the pairs. This can improve the overall estimates of the whole procedure.

Let $m > n$ and additional $m - n$ many $X$-samples $X_{n+1}, \dots, X_{m}$ be given and we consider the case where the PCA map is estimated using all $m$ many $X$-samples.
\begin{theorem}
	\label{prop:mn}
	With the setting and notation of Theorem \ref{thm:overallerror}, granting additionally that Assumption \ref{ass:alphabeta} holds for $\tilde{\rho}$, we have with probability $(1-\eta)$ with respect to $\rho^{\otimes m}$ that
	\begin{align}
	\begin{split} 
	\label{eq:supervisederror}
	&\|\hat{\boldsymbol{f}}_{\lambda, n} \circ \hat{P}_{m} - f_{\tilde{\rho}}\circ P \|^2_{L^2(\rho_X)} \\
	&\qquad\leq\bar C\frac{(1+S_n^2) (\|\hat f_{\lambda,n}\|_{\mathrm{Lip}}^2 + \|\hat{f}_{\lambda, n}\|_\infty^2)}{\eta^2 m\lambda } + C_l \log(9/\eta) \Big(\lambda^{\beta} + \frac{1}{\lambda^\alpha n} \Big).
	\end{split}
	\end{align}%
	We have $\|\hat f_{\lambda,n}\|_{\mathrm{Lip}}^2 + \|\hat{f}_{\lambda, n}\|_\infty^2\le(1+L^2)M^2/\lambda$.
	Assuming \eqref{eq:assnlargeenough}, the right hand side of \eqref{eq:supervisederror} is of order $\frac{1}{\lambda^2 m}+ \lambda^\beta+ \frac{1}{\lambda^\alpha n}$ for large $n, m$ and small $\lambda$.
	If $m > n^{\frac{2+\beta}{\alpha +\beta}}$ the optimized choice $\lambda=n^{-\frac{1}{\alpha+\beta}}$ yields the rate of convergence
	\[
	\|\hat{\boldsymbol{f}}_{\lambda, n} \circ \hat{P}_{m} - f_{\tilde{\rho}}\circ P \|^2_{L^2(\rho_X)} \leq C n^{-\frac{\beta}{\beta + \alpha}}.
	\]
	Otherwise we obtain the rate $m^{-\frac{\beta}{\beta + 2}}$ for $\lambda = m^{-\frac{1}{2+\beta}}$. 
\end{theorem}
For sufficiently large $m$, we thus recover the rate from Lemma~\ref{rem:optimalrate} in the low-dimensional space.

	To complete the picture, Proposition \ref{prop:procedureerror} gives a direct result on controlling the structural error $\|f_\rho - f_{\tilde{\rho}} \circ P\|_{L^2(\rho_X)}$ of the procedure under a Lipschitz condition for $f_\rho$. It shows that the magnitude of the error is controlled by the smallest eigenvalues of the covariance matrix of $X \sim \rho_X$.
\begin{proposition}
	\label{prop:procedureerror}
	Let $f_\rho$ be Lipschitz continuous with constant $L$. Then it holds
	\[
	\| f_\rho - f_{\tilde{\rho}} \circ P \|_{L^2(\rho_X)} \leq 2 L \Big(\sum_{i=d+1}^D \sigma_i^X\Big)^{1/2},
	\]
	where $\sigma_1^X, \dots \sigma_D^X$ are the eigenvalues of the covariance matrix of $X \sim \rho_X$ in decreasing order.
\end{proposition}

\section{Numerical Examples}
\label{sec:examples}

In this section, the overall procedure of dimensionality reduction and kernel regression is illustrated in a simple setting. We mainly aim to shed some light on the comparison of the occurring errors for direct regression versus inclusion of dimensionality reduction, as discussed above. Due to computational constraints, the insight of these examples towards true asymptotic behavior is of course limited, but we believe the illustration for smaller sample sizes and visualization of the absolute errors involved can nevertheless be insightful. 

Let $d = 2$ and $D = 10$. The $X$-data is generated as follows: Consider $\tilde{X}$ to be uniformly distributed on $[-1, 1]^d \times [-\varepsilon, \varepsilon]^{D-d}$, where we set $\varepsilon = 0.1$. We define $X$ to be some rotation (with an orthogonal transformation) of $\tilde{X}$, say $X = A\tilde{X}$. The first two eigenvalues of the covariance matrix of $X$ are thus $1/3$ and the remaining ones are $1/30$.

As in Section \ref{sec:PCA}, we denote by $P$ the PCA map and $\hat{P}_n$ the estimated counterpart.
The excess reconstruction error for the PCA procedure depending on the number of sample points used, which is studied in Lemma \ref{lem:pcaerror}, is showcased in Figure \ref{fig:ex1_pca}. The shown error, after an application of Markov's inequality, yields Lemma \ref{lem:pcaerror} which governs the term $\varepsilon$ in the stability result in Proposition \ref{prop:pcakernelstability}. We see that the (predicted) linear behavior in $n$ arises.

\begin{figure}
	\centering
	\includegraphics[width=0.50\textwidth]{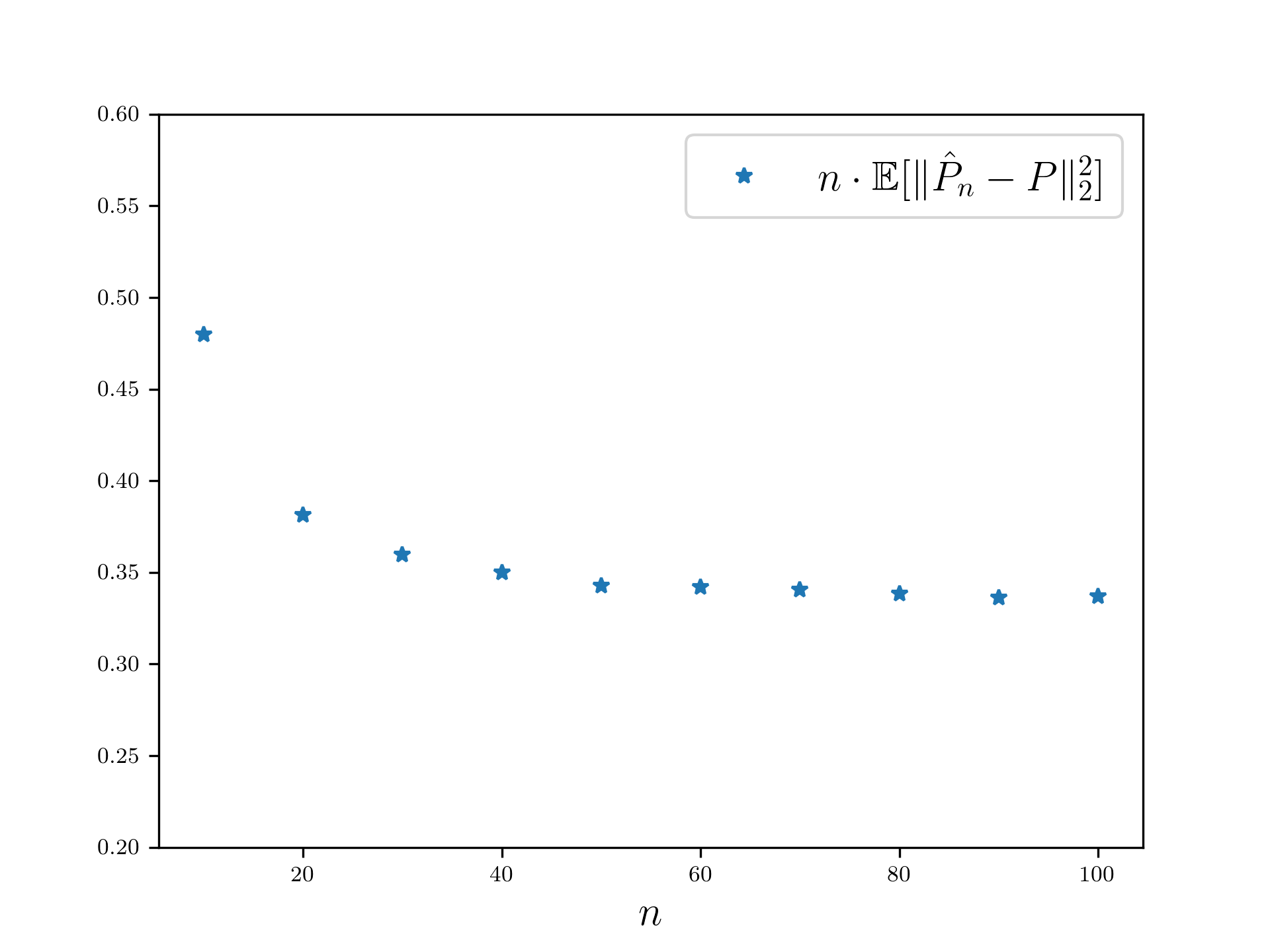}
	\caption{Rate for the excess reconstruction error (the error occurring from sample estimation of the PCA map) for the PCA map in the Example given in Section \ref{sec:examples}. To estimate the expectation, we used the Monte-Carlo method with $10000$ sample points for $\hat{P}_n$ for each $n$. }
	\label{fig:ex1_pca}
\end{figure}

For the $Y$-data, we treat two different cases. Roughly speaking, the cases differ in the sense of whether the dependence structure of $(X, Y)$ includes the noise inherent in $X$ or not. For both cases, we fix a function $f^{(1)}\colon \mathbb{R}^D \rightarrow \mathbb{R}$, which we (rather arbitrarily) define by $f^{(1)}(x) := \sin\big(\sum_{i=1}^D x_i\big)$.

\begin{figure}
	\includegraphics[width=0.5\textwidth]{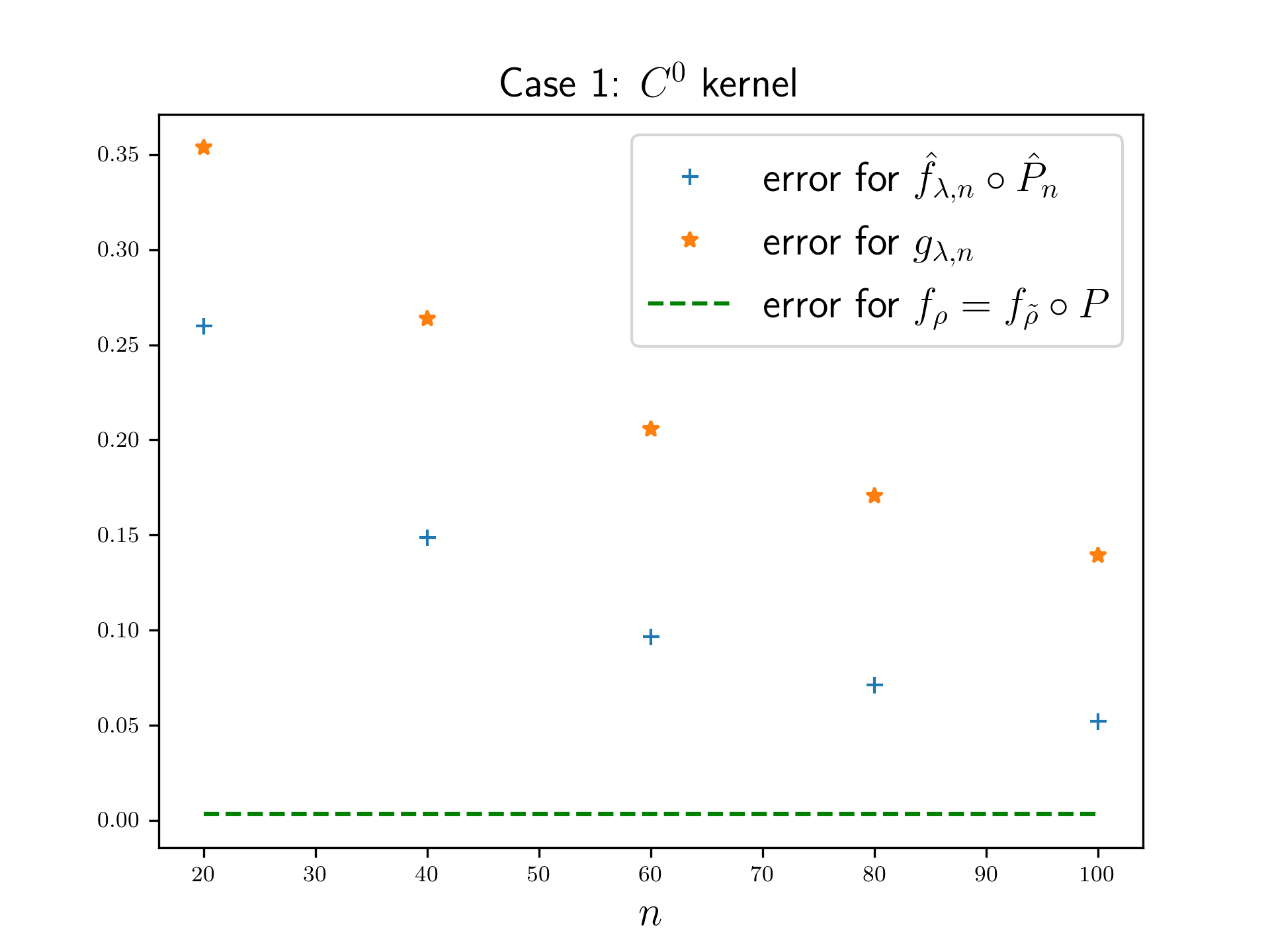}	
	\includegraphics[width=0.5\textwidth]{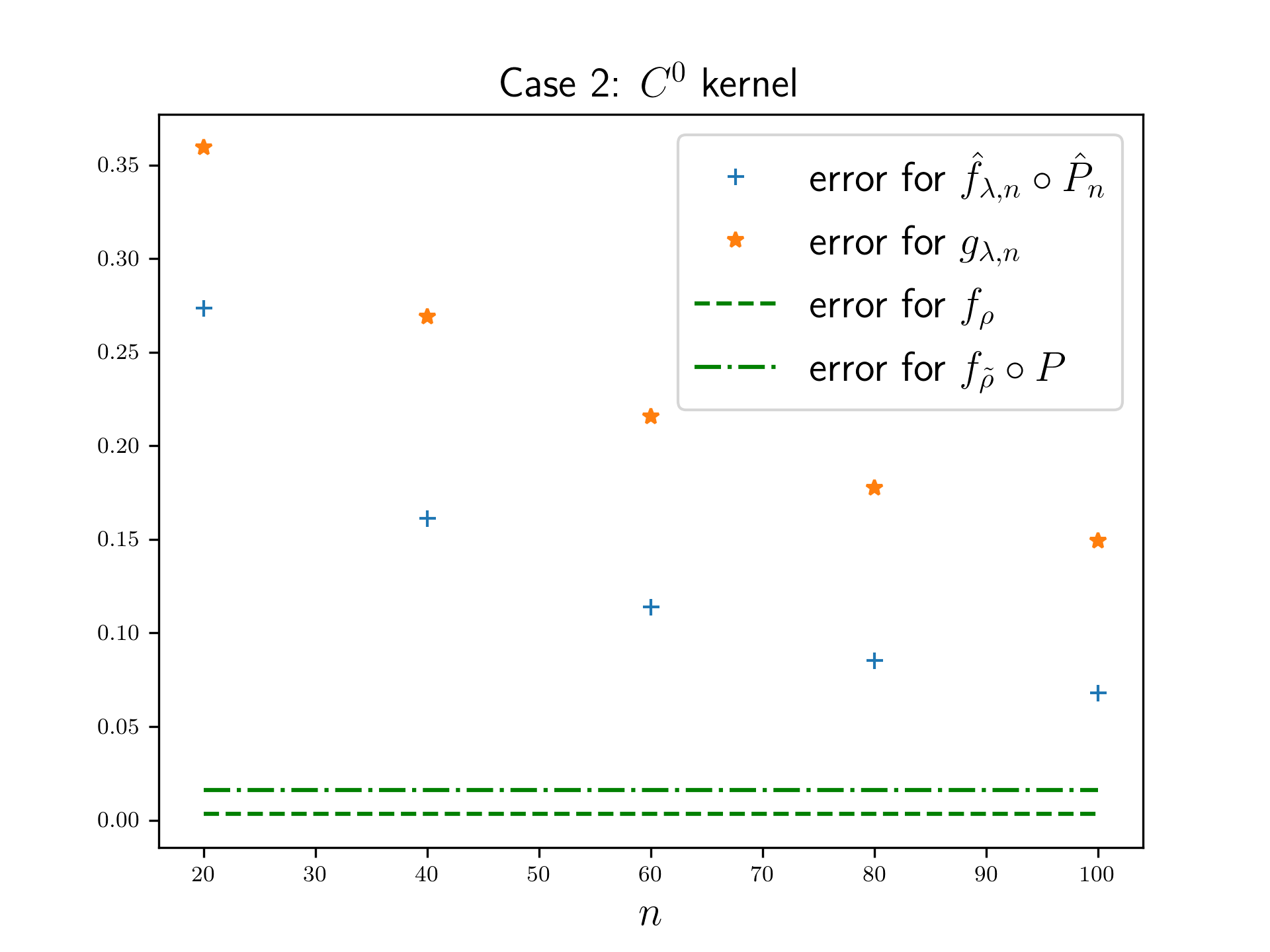}	
	\includegraphics[width=0.5\textwidth]{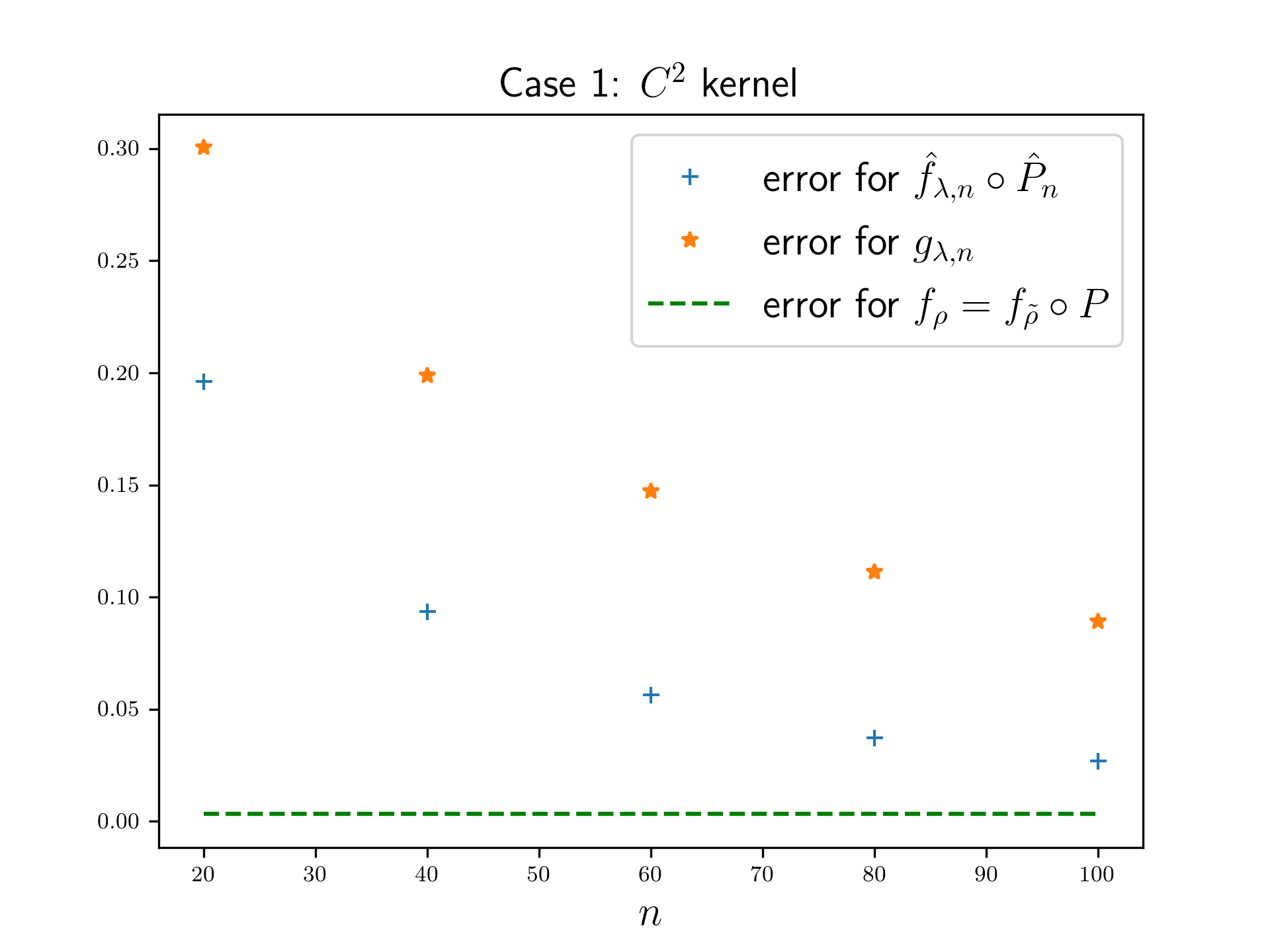}
	\includegraphics[width=0.5\textwidth]{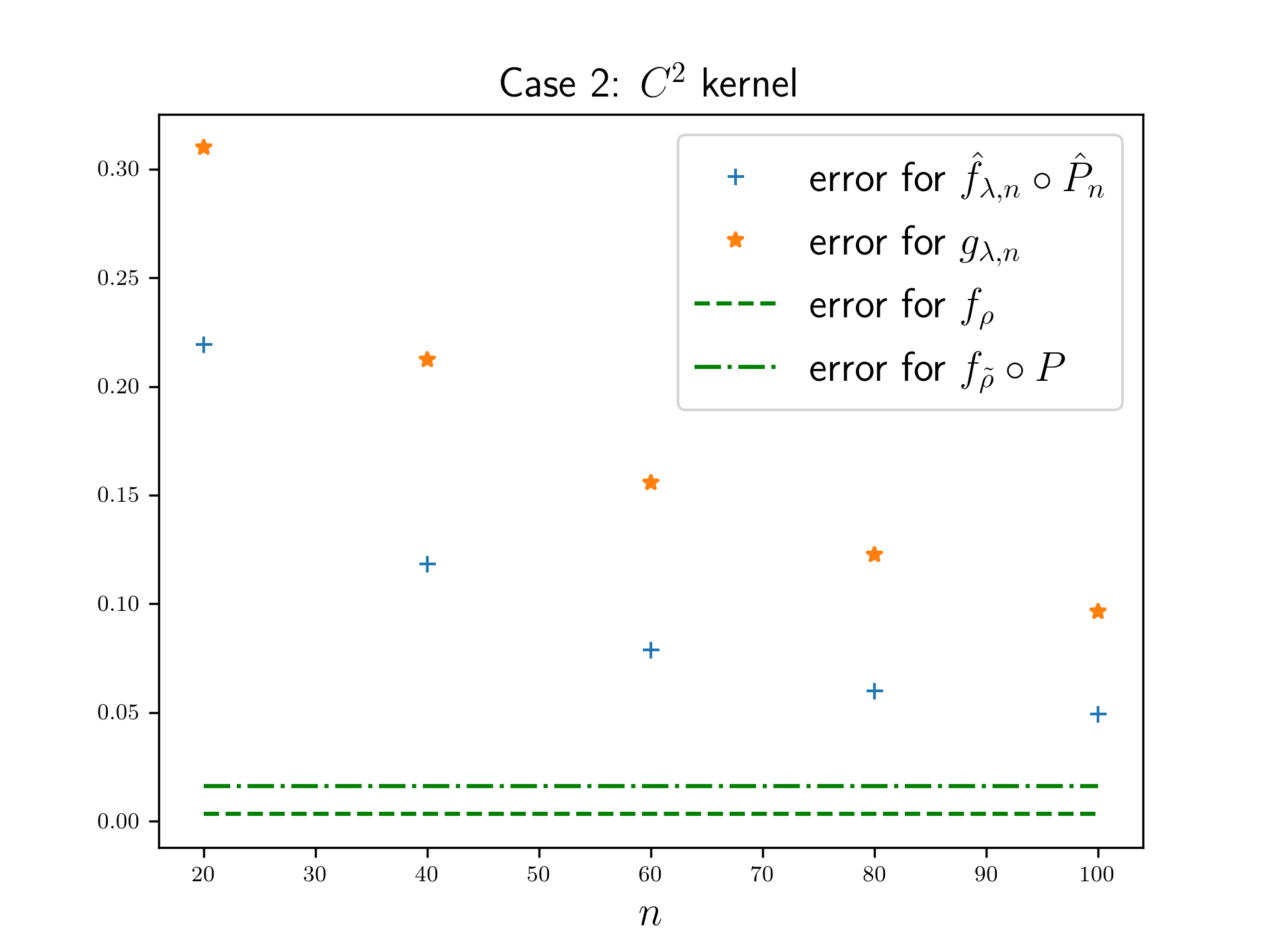}
	\includegraphics[width=0.5\textwidth]{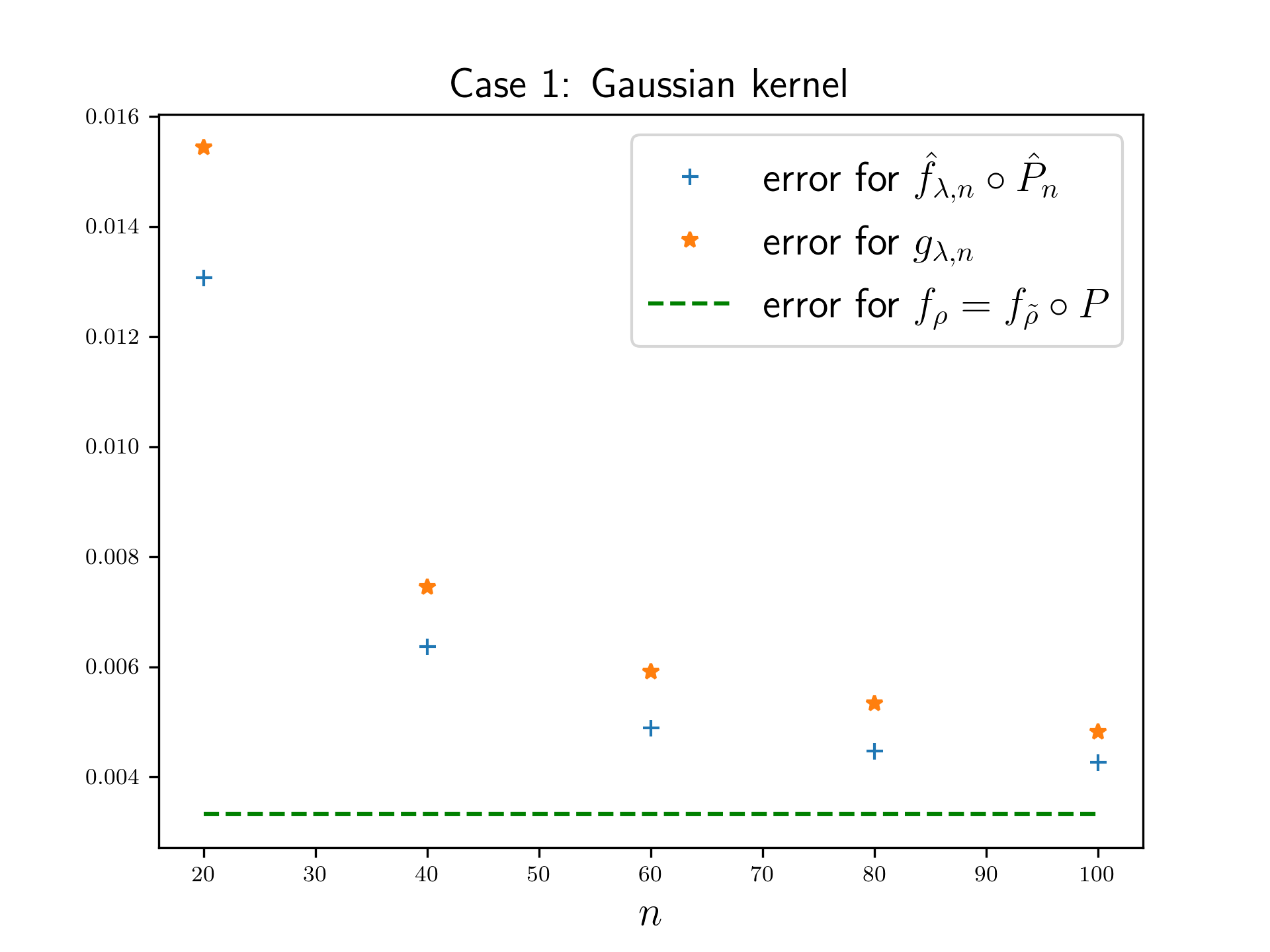}
	\includegraphics[width=0.5\textwidth]{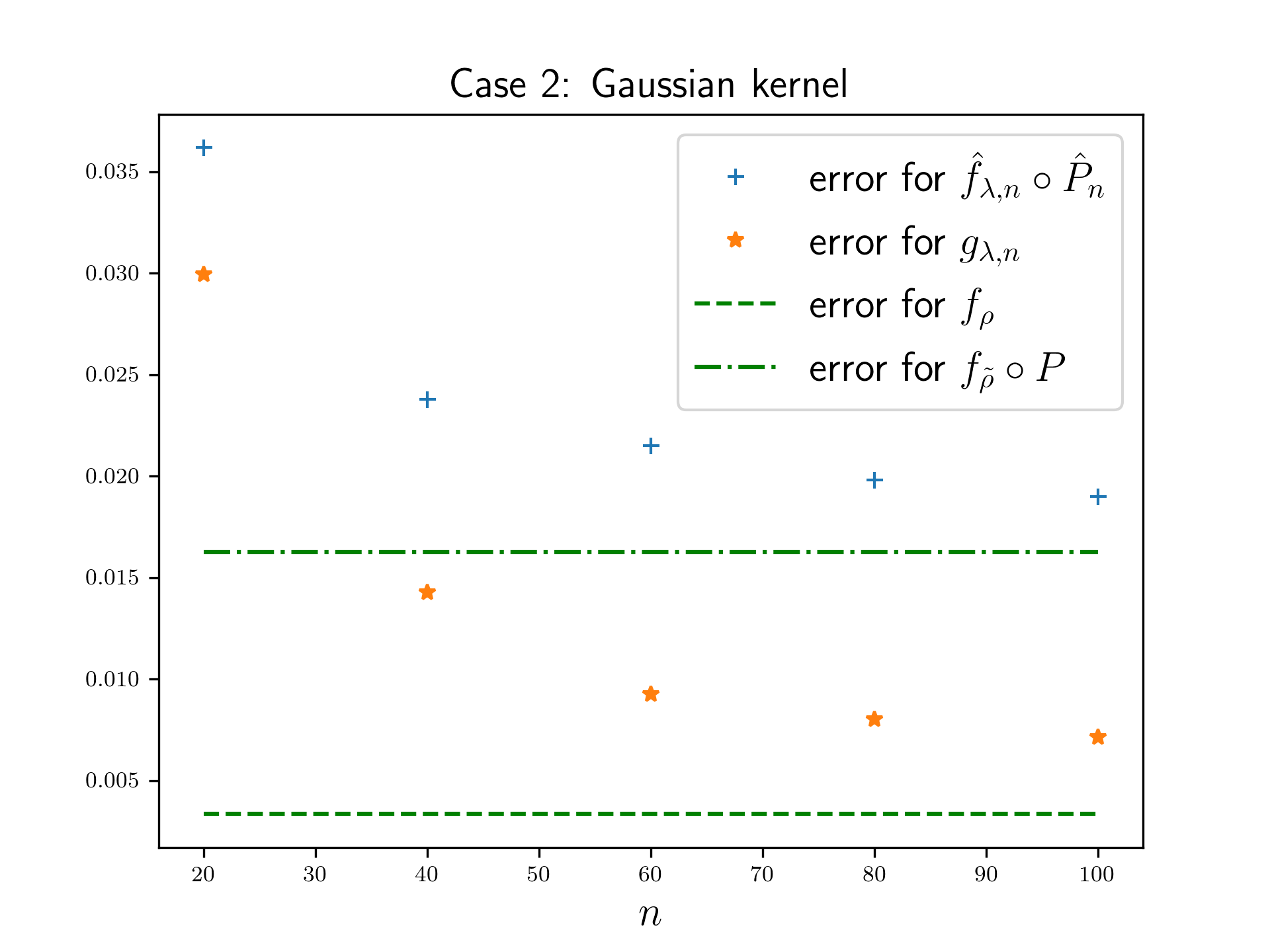}
	\caption{Overall estimation error $\int (\hat f(x) - y)^2 \rho(dx, dy)$ for different estimators $\hat f$ and different generation of $Y$ data. The different estimators are the asymptotic estimates $f_\rho$ and $f_{\tilde{\rho}} \circ P$, the direct kernel estimator $g_{\lambda, n}$, and the estimator arising from the PCA and subsequent kernel regression procedure $\hat{f}_{\lambda, n} \circ \hat{P}_n$. The respective kernels are described in \eqref{eq:kernelchoices} and $\lambda$ is set via cross-validation. Each reported value is an average over 50 independent runs of generating the respective sample.}
	\label{fig:case1and2}
\end{figure}

For the first case, we define $Y := f^{(1)}(P(X)) + U$, and for the second case $Y := f^{(1)}(X) + U$, where $U$ is uniformly distributed on $[-0.1, 0.1]$ and independent of all other variables. We see that in the first case, the dependence structure between $X$ and $Y$ is completely mediated by the first two principal components of $X$, while in the second case all principal components of $X$ are relevant. We thus expect the dimensionality reduction method to be particularly useful in the first case. 

We use cross-validation for the choice of $\lambda$ as described in \cite{steinwart2009optimal}, which is therein shown to be optimal asymptotically and thus leads to a regime of $\lambda$ where the error bound from Corollary \ref{cor:errorsummary} applies.
%
The results are presented for different kernel functions varying in their regularity. Each kernel is of the form $K(x, y) = \phi(|x-y|)$, where we use
\begin{equation}
\label{eq:kernelchoices}
\begin{alignedat}{2}
\phi_\infty(r) &= \exp(-r^2)~~~ &&\text{($C^\infty$, \textsl{Gaussian kernel})},\\
\phi_2(r) &= \max\{0, 1-r\}^8 \, (8 r +1)~~~ &&\text{\citep[$C^2$, cf.][Table 4.1]{zhu2012compactly}},\\
\phi_0(r) &= \max\{0, 1-r\}^6 \, (8 r +1)~~~ &&\text{\citep[$C^0$, cf.][Table 4.1]{zhu2012compactly}}.\\
\end{alignedat}
\end{equation}

We observe the estimation errors for the first case and different choices of kernel functions in Figure \ref{fig:case1and2}. As expected by the data generating process for the first case, the procedure including the dimensionality reduction step performs very well. However, the dimensionality reduction also mostly reduces the total error in the second case (for the $C^0$ and $C^2$ kernels), while only in the Gaussian case (where the absolute error is already very small), the asymptotic error dominates and the procedure involving dimensionality reduction produces larger absolute error, see also the discussion after Corollary~\ref{cor:errorsummary}, which gives an explanation for this behavior.\footnote{We note another possible reason for the difference in absolute errors between the choices of kernel can be the bandwidth. This means, the absolute errors may be more similar if we tried to optimize the bandwidth suitably (i.e., work with kernels $r \mapsto \phi(\gamma r)$ for suitable choices of $\gamma > 0$).} 

\begin{figure}
	\includegraphics[width=0.5\textwidth]{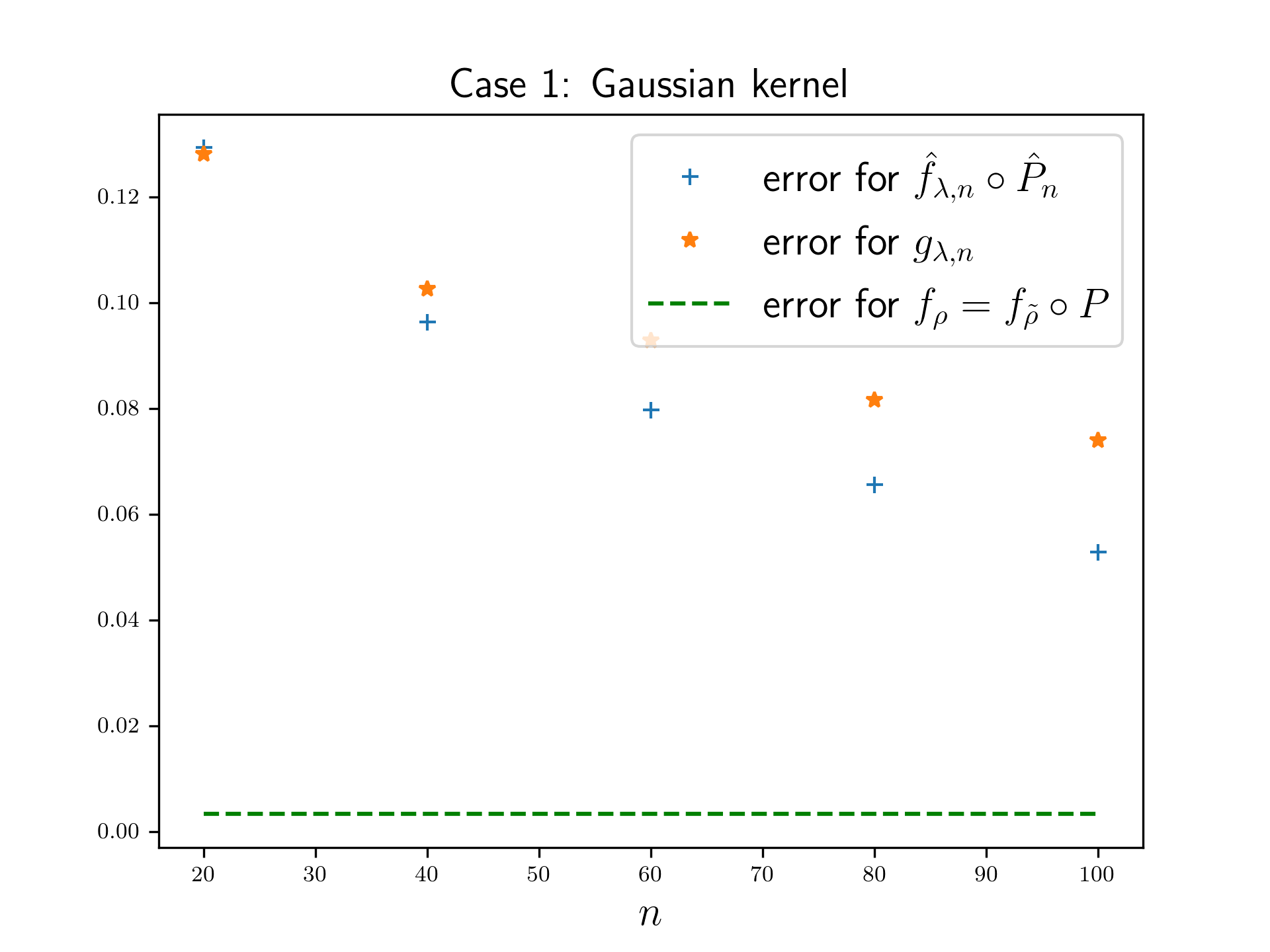}	
	\includegraphics[width=0.5\textwidth]{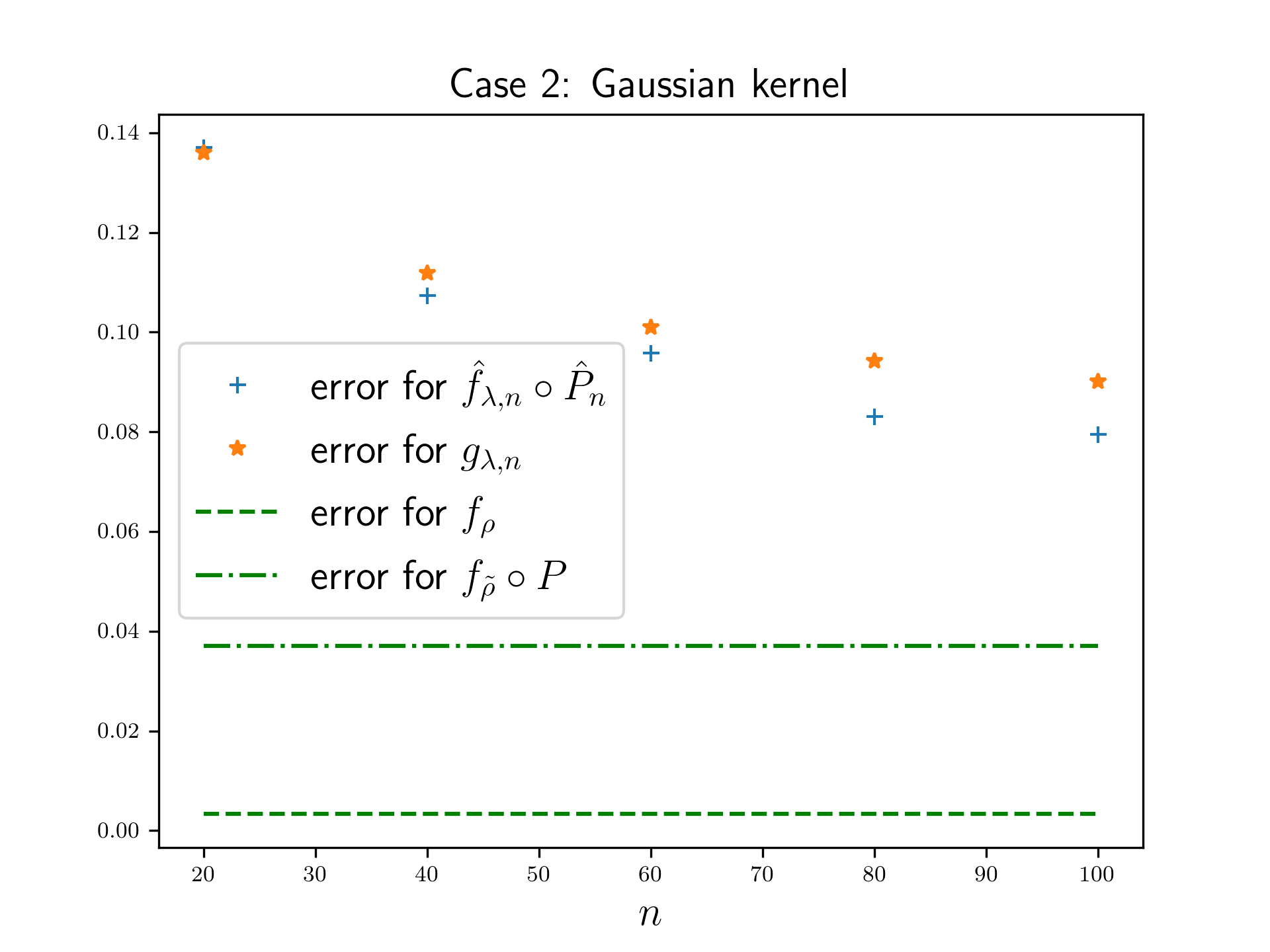}
	\includegraphics[width=0.5\textwidth]{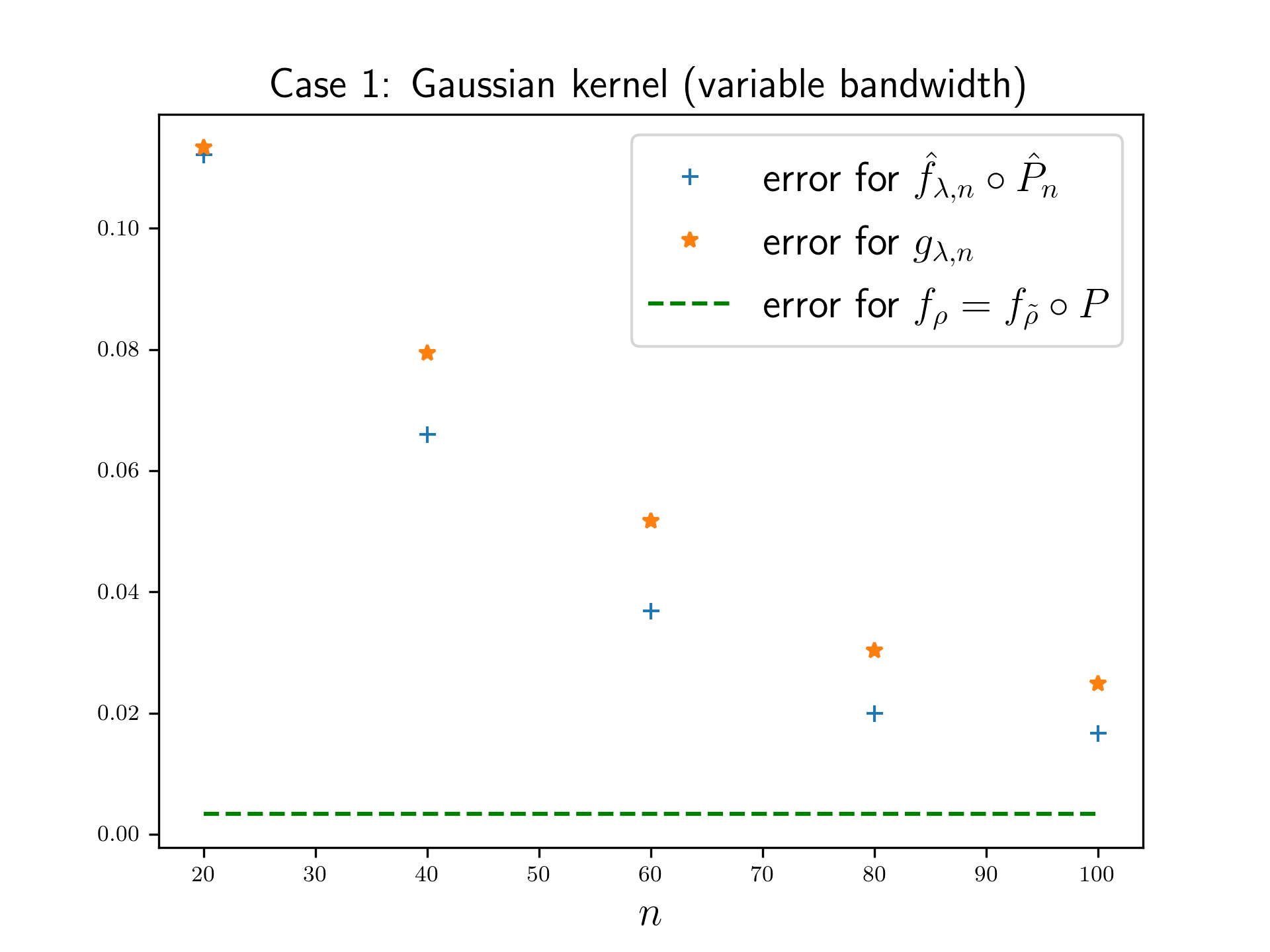}	
	\includegraphics[width=0.5\textwidth]{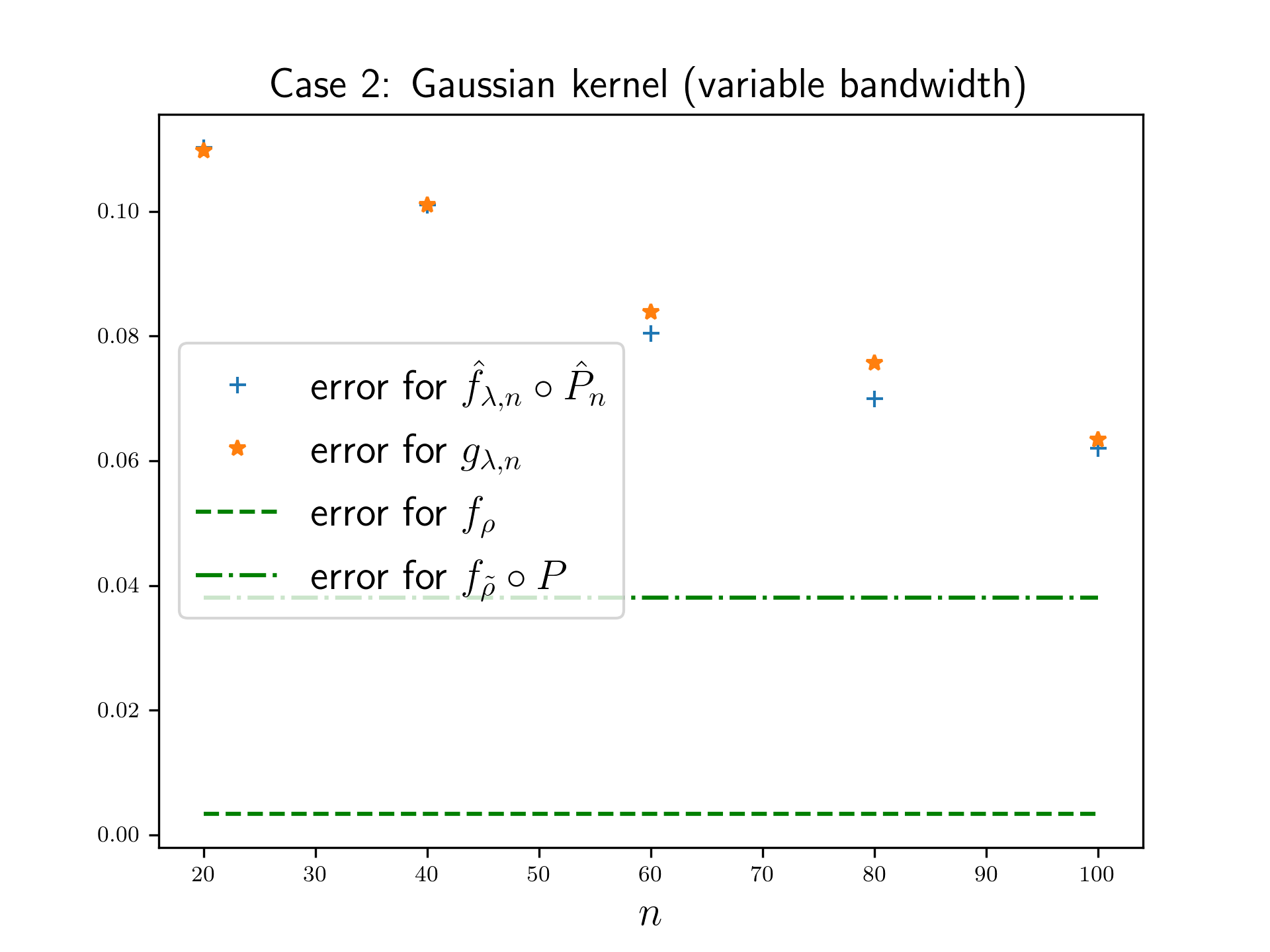}
	\includegraphics[width=0.5\textwidth]{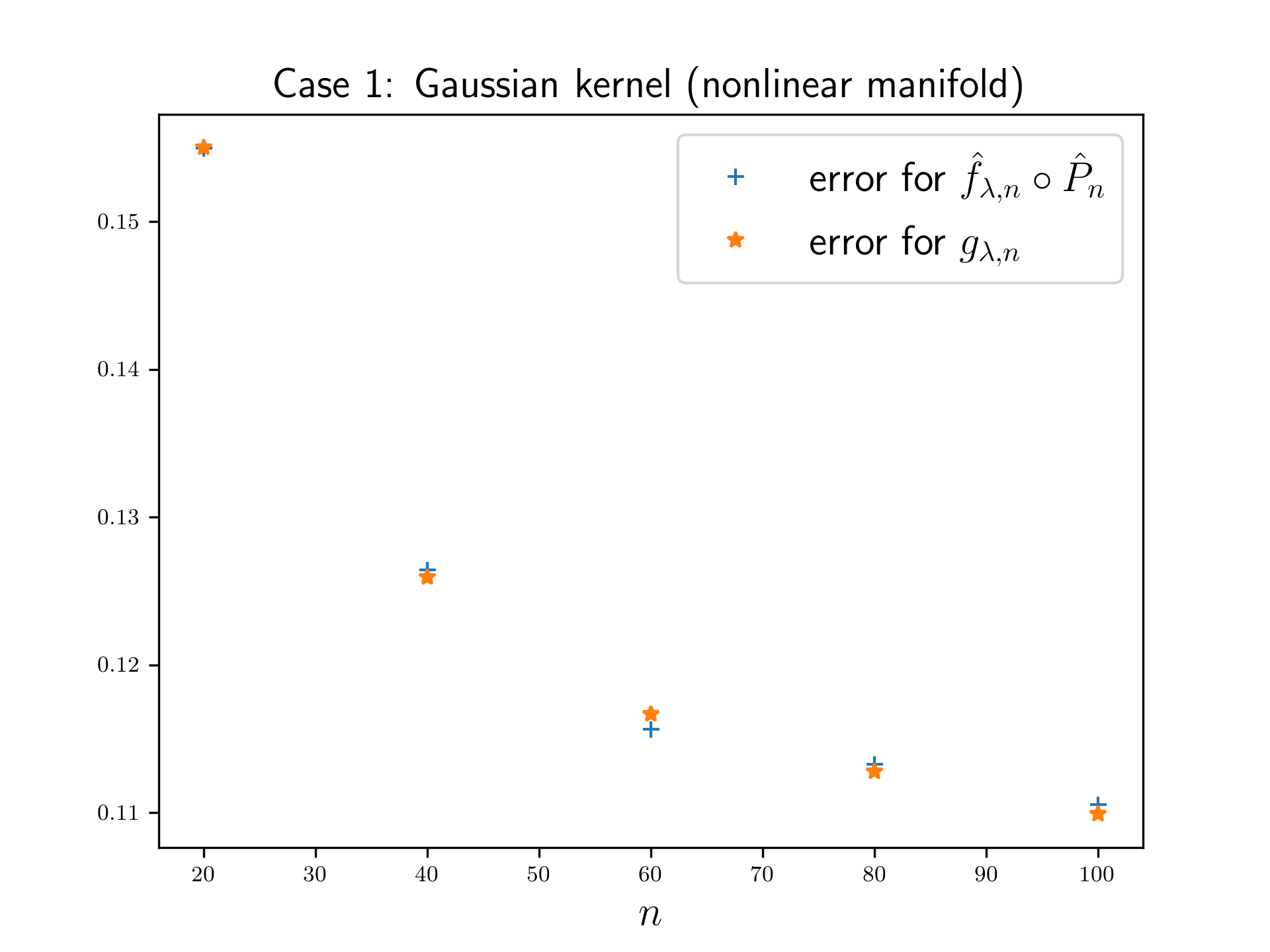}	
	\includegraphics[width=0.5\textwidth]{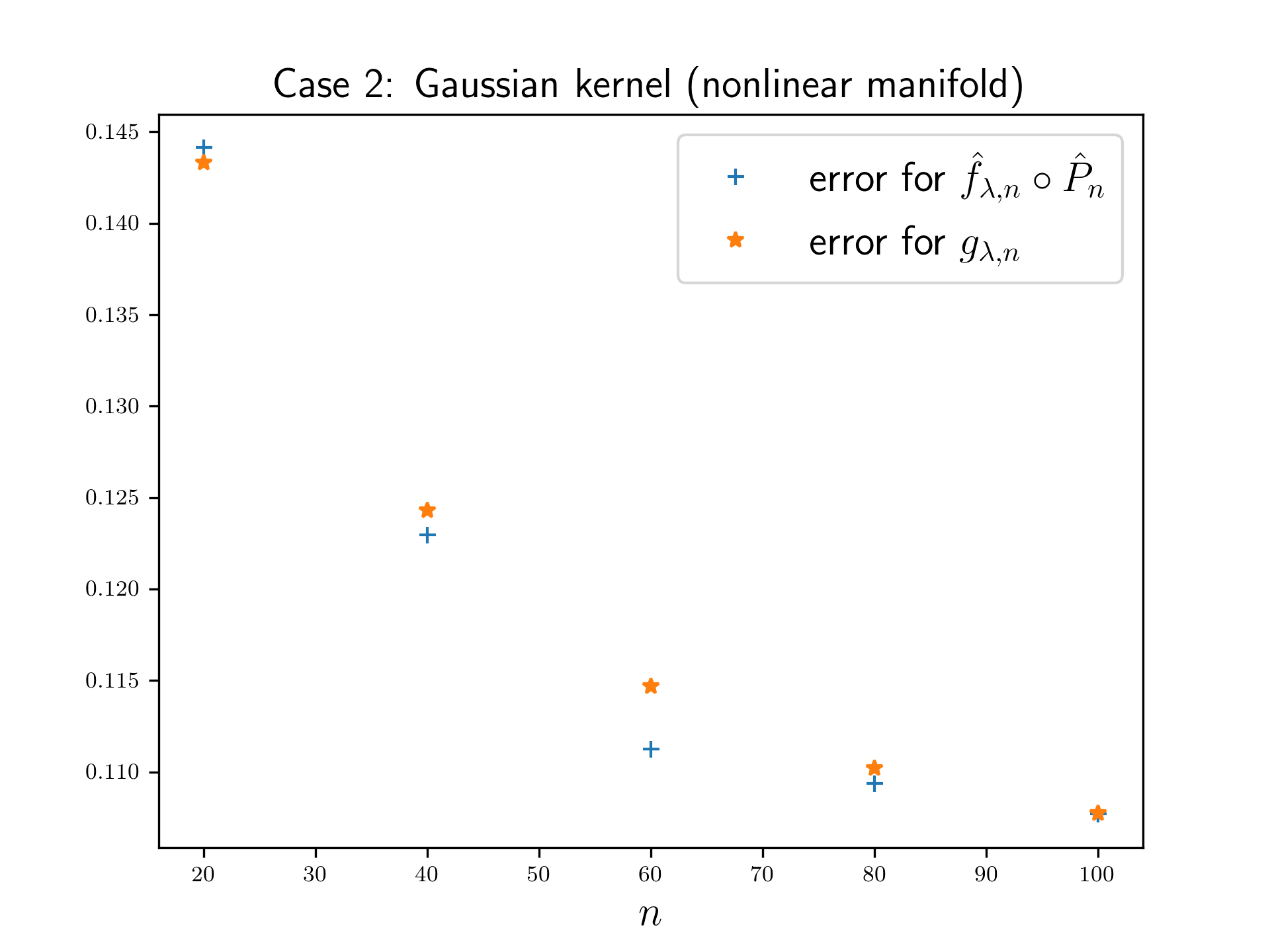}
	\caption{Overall estimation error $\int (\hat f(x) - y)^2 \rho(dx, dy)$ for different estimators $\hat f$ where the data is determined via the non-differentiable function $f^{(2)}$. Each reported value is an average over 50 independent runs of generating the respective sample.  The different estimators are the asymptotic estimates $f_\rho$ and $f_{\tilde{\rho}} \circ P$, the direct kernel estimator $g_{\lambda, n}$, and the estimator arising from the PCA and subsequent kernel regression procedure $\hat{f}_{\lambda, n} \circ \hat{P}_n$. \emph{First line:} Gaussian kernel $\phi(\lambda r) = \exp(-h r^2)$ with fixed bandwidth $h=1$ and cross-validated $\lambda$. \emph{Second line:} The bandwidth $h\in [10^{-3}, 10]$ is chosen with cross-validation, too. \emph{Third line:} Features are generated via $X = A\tilde{X} + 0.25 \sin(\sum_{i=1}^{10} \tilde{X}_i)$ (instead of $X = A\tilde{X}$). }
	\label{fig:gaussian_irregular_fun}
\end{figure}

Since the function $f^{(1)}(x) := \sin(\sum_{i=1}^D x_i)$ defining the dependence structure of $(X, Y)$ is $C^\infty$, it is natural that the Gaussian kernel performs well. It is interesting to observe what happens when a less regular function is used. To this end, we define $f^{(2)}(x) := |\sin(2\sum_{i=1}^D x_i)|$ and $(X, Y)$ for the two different cases analogously to the above. The top two graphs of Figure \ref{fig:gaussian_irregular_fun} report the results. Obviously, all occurring errors are larger and the observed rate of convergence is slower compared to the data generated via $f^{(1)}$, compatible with Lemma \ref{rem:optimalrate} and Corollary \ref{cor:errorsummary}. Since the Gaussian kernel does not quickly convergence to the best estimator, the effect of the dimensionality reduction is visible even for  Case 2. A plausible reason is that the irregular relation of $(X, Y)$ is smoothed via the dimensionality reduction, and thus the parameter $\beta$ of Assumption \ref{ass:alphabeta} is larger (leading to faster convergence) for the two-step procedure (c.f.~the discussion after Corollary \ref{cor:errorsummary}). In Figure \ref{fig:gaussian_irregular_fun}, we further showcase the behavior that occurs when we use cross-validation for the bandwidth parameter for the Gaussian kernel (the two graphs in the middle) and when generating the $X$-data using a nonlinear transformation (the bottom two graphs in Figure \ref{fig:gaussian_irregular_fun}). While cross-validation of the bandwidth is an important practical tool, in the literature on learning rates of kernel regression, one usually fixes the kernel, which is also the case in this paper. Nevertheless, we see in Figure \ref{fig:gaussian_irregular_fun} that the observed patterns still appear similar to the previous cases even if cross-validation for the bandwidth is used. Finally, as expected, when the $X$-data was generated via a nonlinear transformation, the method using PCA performs generally worse compared to previous cases. Overall, we believe the study of nonlinear manifolds using also nonlinear dimensionality reduction methods is an interesting avenue for further research.

\begin{figure}
	\includegraphics[width=0.5\textwidth]{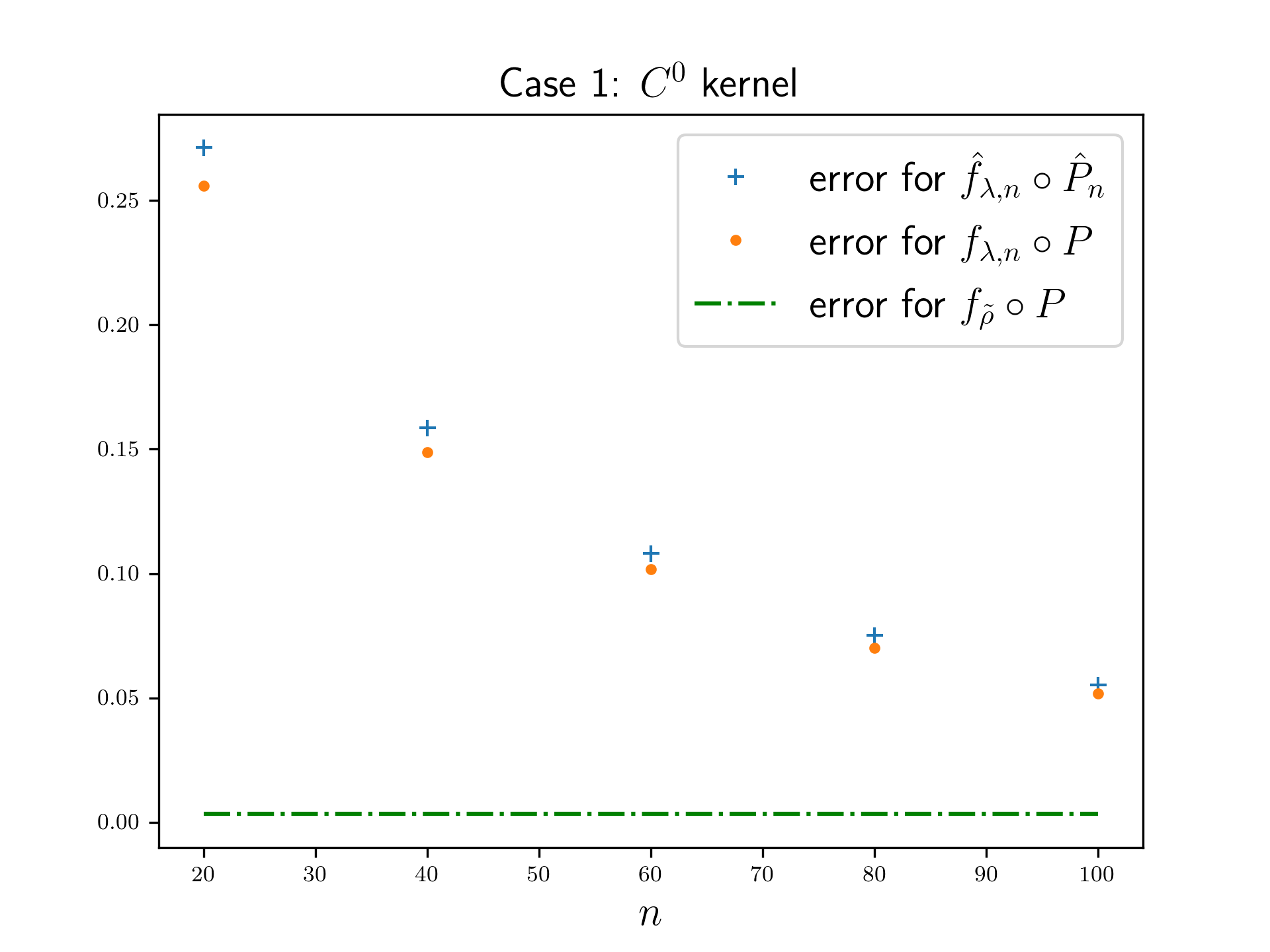}	
	\includegraphics[width=0.5\textwidth]{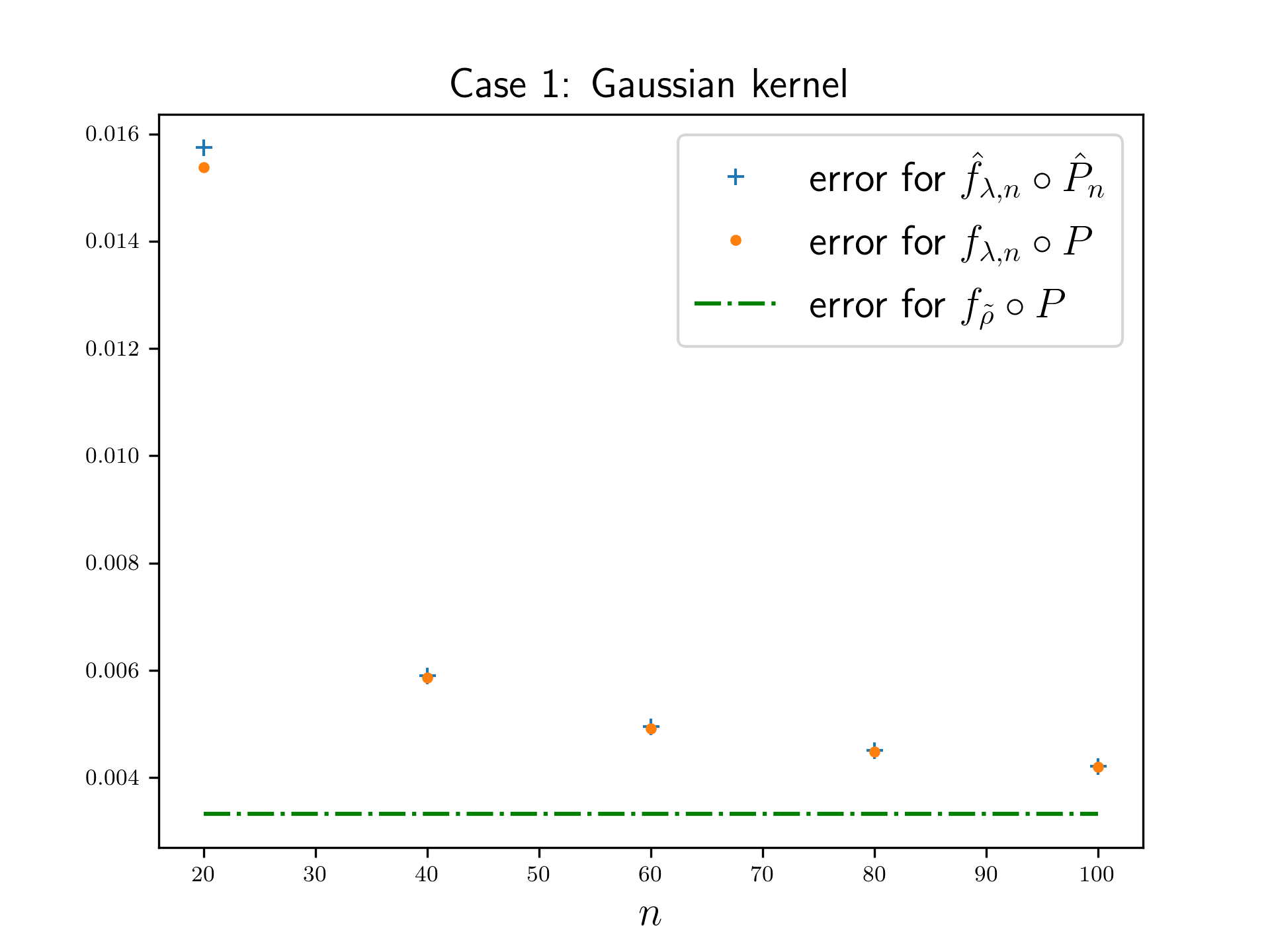}
	\begin{center}\includegraphics[width=0.5\textwidth]{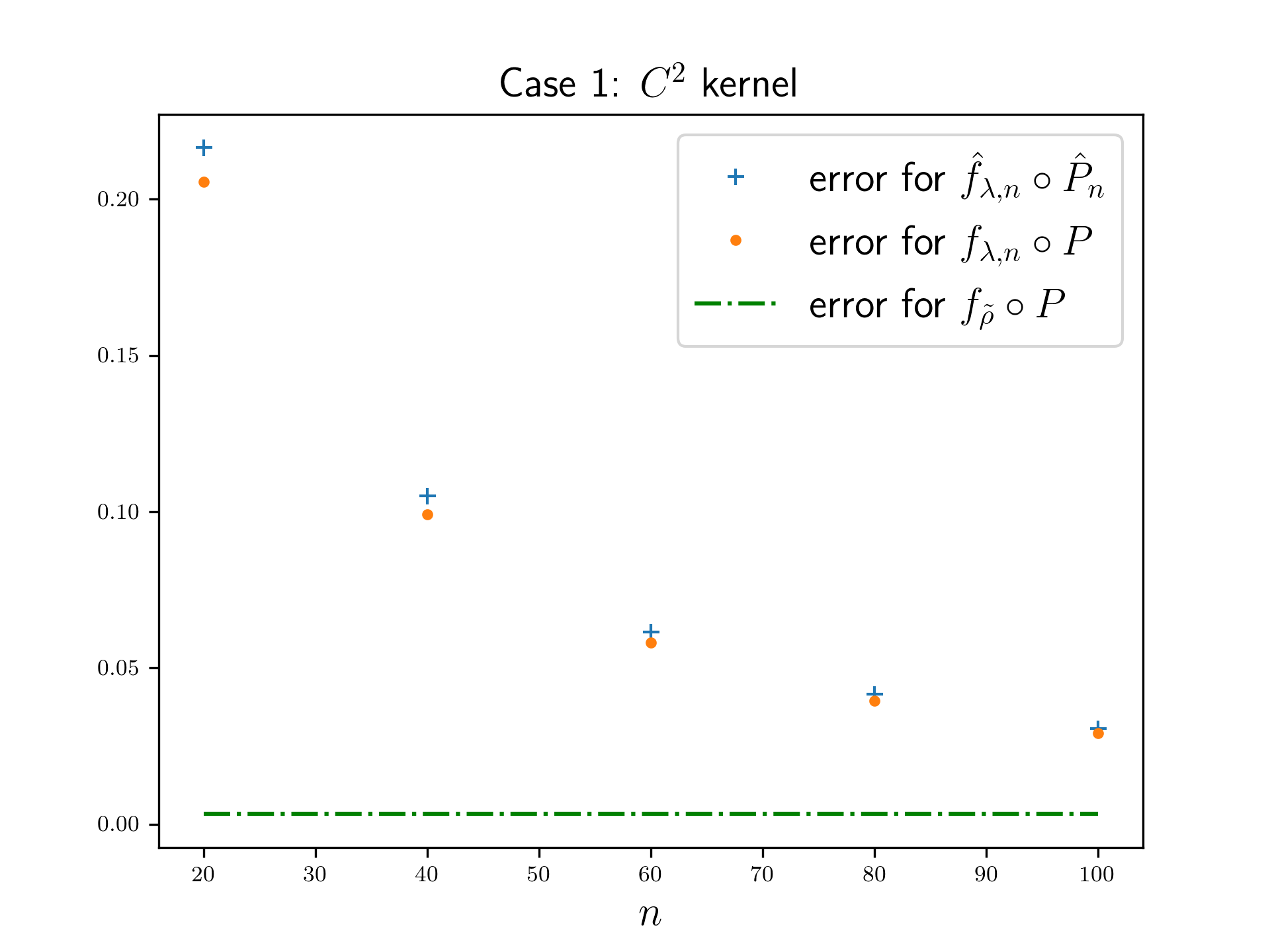}
	\end{center}
	\caption{Overall estimation error $\int (f(x) - y)^2 \rho(dx, dy)$ for different estimators $f$. The different estimators are the asymptotic estimator $f_{\tilde{\rho}} \circ P$, the estimator arising from the PCA and subsequent kernel regression procedure $\hat{f}_{\lambda, n} \circ \hat{P}_n$, as well as the estimator $f_{\lambda, n}^* \circ P$ which uses samples but assumes knowledge of the true PCA map $P$. The respective kernels are described in \eqref{eq:kernelchoices} and $\lambda$ is set via cross-validation. Each reported value is an average over 50 independent runs of generating the respective sample.}
	\label{fig:pcaoracle}
\end{figure}

To emphasize the effect of estimating the dimensionality reduction map instead of starting in the low-dimensional space directly, Figure \ref{fig:pcaoracle} shows the difference between the estimator $\hat{f}_{\lambda, n}$ which uses the estimated PCA map $\hat{P}_n$ and the estimator $f_{\lambda, n}^*$ using the true PCA map $P$ (for the first case described above, where $Y = f(P(X)) + U$). While we see that the estimation of the PCA map does influence the total error in a non-negligible way, it is not clear whether this influence is really relevant asymptotically. This suggests that the estimates of Theorems \ref{thm:stab} and Proposition \ref{prop:pcakernelstability} may be improved upon for situations as the one treated in this example, perhaps under the assumption of higher order differentiability of the kernel function, instead of just the Lipschitz condition required in Assumption \ref{ass:kernel}.

	\section{Proofs}
\label{sec:proofs}	
\subsection{Proofs for Section \ref{sec:kernel}}
\begin{proof}\textbf{of Lemma \ref{lem:ortho}}
	Before showing the individual claims, we first make the following observation: Take $\lambda > 0$ and $\varphi\colon \mathcal{X} \rightarrow \mathbb{R}$. Then $(\lambda, \varphi)$ is an eigenvalue/-function pair $(\lambda, \varphi)$ of $L_{K, \rho_X}$ if and only if $(\lambda, \varphi \circ A_{\rm inv})$ is one of $L_{K, \tilde{\rho}_{\tilde{X}}}$.
	This follows because $A_{\rm inv}$ has orthonormal columns and thus $|\tilde{x}| = |A_{\rm inv}(\tilde{x})|$ for all $\tilde{x} \in \tilde{\mathcal{X}}$. This implies
	\begin{align*}
	L_{K, \rho_X}\varphi(x) &= \int \phi(|x-u|) \varphi(u) \rho_X(du) \\
	&= \int \phi(|A_{\rm inv}(A(x) - \tilde{u})|) \varphi \circ A_{\rm inv}(\tilde{u}) \tilde{\rho}_{\tilde{X}}(d \tilde{u}) \\
	&= \int \phi(|A(x) - \tilde{u}|) \varphi \circ A_{\rm inv}(\tilde{u}) \tilde{\rho}_{\tilde{X}}(d \tilde{u}) = L_{K, \tilde{\rho}_{\tilde{X}}} \varphi \circ A_{\rm inv}(A(x))
	\end{align*}
	and hence if $L_{K, \rho_X} \varphi(x) = \lambda \varphi(x)$ then 
	\[
	L_{K, \tilde{\rho}_{\tilde{X}}} (\varphi \circ A_{\rm inv})(A(x)) = L_{K, \rho_X} \varphi(x) = \lambda \varphi(x) = \lambda (\varphi \circ A_{\rm inv})(A(x))
	\]
	and vice versa.
	
	(i): We first show that $f_{\rho} = f_{\tilde{\rho}} \circ A$. To this end, we show that $\rho(\cdot | x) = \tilde\rho(\cdot | A(x))$. First note that $\tilde{\rho}_X = \rho_X\circ A^{-1}$ ($\tilde{\rho}_X$ is the pushforward measure of $\rho_X$ under the map $A$) and $\rho_X= \tilde{\rho}_X \circ A_{\rm inv}^{-1}$. For a continuous and bounded function $g: \mathcal{Z} \rightarrow \mathbb{R}$, it holds
	\begin{align*}
	\int g \,d\rho &= \int g(A_{\rm inv}(A(x)), y) \,\rho(dx, dy) \\
	&= \int\int g(A_{\rm inv}(\tilde{x}), y) \,\tilde\rho(dy | \tilde{x}) \tilde\rho_{\tilde{X}}(d\tilde{x}) \\
	&= \int \int g(A_{\rm inv}(\tilde{x}), y) \,\tilde\rho(dy | A(A_{\rm inv}(\tilde{x}))) \tilde\rho_{\tilde{X}}(d\tilde{x})\\
	&= \int \int g(x, y) \,\tilde\rho(dy | A(x)) \rho_X(dx),
	\end{align*}
	which shows $\rho(\cdot | x) = \tilde{\rho}(\cdot | A(x))$.
	
	(ii): We show $f_{\lambda, \rho} = f_{\lambda, \tilde{\rho}} \circ A$. By the relation of the eigenvalue/-function pairs of $L_{K, \rho_X}$ and $L_{K, \tilde{\rho}_{\tilde{X}}}$ as shown at the beginning of the proof, it follows from Mercer's theorem \citep[see, e.g.,][]{steinwart2012mercer} that $\|f\|_{\mathcal{H}} = \|f \circ A\|_{\mathcal{H}}$ holds for all $f \in \mathcal{H}$ mapping from $\tilde{X}$ to $\mathbb{R}$. Further, 
	\[
	\int (f(\tilde{x})-y)^2 \,\tilde\rho(d\tilde{x}, dy) = \int (f(A(x))-y)^2 \,\rho(dx, dy)
	\]
	and thus $f_{\lambda, \tilde{\rho}} \circ A$ is the unique optimizer of 
	\[
	\argmin_{f \in \mathcal{H}} \int (f(x) - y) \,\rho(dx, dy) + \lambda \|f\|_{\mathcal{H}}^2
	\]
	which yields $f_{\lambda, \tilde{\rho}} \circ A = f_{\lambda, \rho}$.
	
	(iii): The argument given at the beginning of the proof immediately implies the same eigenvalue decay of the operators $L_{K, \rho_X}$ and $L_{K, \tilde{\rho}_{\tilde{X}}}$ and thus the parameter $\alpha$ can be chosen equally. Using (i) and (ii), we have
	\[
	\int (f_{\lambda, \rho} - f_\rho)^2 \,d\rho = \int (f_{\lambda, \tilde{\rho}}\circ A - f_{\rho} \circ A)^2 \,d\rho = \int (f_{\lambda, \tilde\rho} - f_{\tilde{\rho}})^2 \,d\tilde{\rho}.
	\]
	Therefore, the parameter $\beta$ can be chosen equally, too. 
\end{proof}

\begin{proof}\textbf{of Lemma \ref{lem:noise}}
	We show that $c \leq a$, while $c \leq b$ will follow by symmetry. For $z \in \mathbb{R}^D$, define $\rho_X^{z}$ as the distribution $\rho_X$ shifted by $z$, i.e., $\rho_X^z(A) = \rho_X(A+z)$ for all Borel sets $A$. An elementary calculation yields
	\[
	L_{K, \rho_X * \kappa}(f) = \int L_{K, \rho_X^z} (f) \,\kappa(dz).
	\]
	Note that for all $z\in\R^D$ the operators $L_{K, \rho_X^z}$ have the same eigenvalues as $L_{K, \rho_X}$, namely $\sigma_1, \sigma_2, \dots$. 
	By linearity of the trace this shows 
	\begin{equation}\label{eq:infsum}
	\sum_{i=1}^\infty \hat{\sigma}_i = \sum_{i=1}^\infty \sigma_i.
	\end{equation}
	Consider now an empirical approximation
	$\kappa^m = \frac{1}{m}\sum_{i=1}^m \delta_{z_i}$ for certain $z_i \in \mathbb{R}^D$ and let $L_m := \int L_{K, \rho_X^z} \kappa^m(dz)$. Denote the eigenvalues of $L_m$ in decreasing order by $u^m_1, u^m_2, \dots$. 
	Using \cite[Theorem 2]{wielandt1955extremum}\footnote{Note while the Theorem is only stated for finite dimensional operators, it is mentioned in the introduction that the techniques extend to self-adjoint compact operators on a Hilbert space, which fits our setting. See also the related discussion in \cite{fulton2000eigenvalues}.}, we obtain
	\begin{equation}\label{eq:finitesumsineq}
	\sum_{i=1}^r u_i^m \leq \sum_{i=1}^r \sigma_i.
	\end{equation}
	Consider $\Delta_m = L_{K, \rho_X * \kappa} - L_m$. Since $\kappa$ has finite first moment, we can choose suitable $z_1, \dots, z_m$ such that $W_1(\kappa, \kappa^m) \rightarrow 0$ for $m\rightarrow \infty$. We consider $\Delta_m$ as an operator from $\mathcal{H}$ to $\mathcal{H}$, since this does not change its eigenvalue behavior, c.f.~\cite[Remark 2]{caponnetto2007optimal}. Since $K$ is bounded by 1 and Lipschitz continuous with constant $L$, we obtain for $f \in \mathcal{H}, \|f\|_\mathcal{H}=1$ that
	\[
	\|L_{K, \rho_X^z}(f) - L_{K, \rho_X^{\tilde{z}}}(f)\|_{\cal H} \leq 2 L |z-\tilde{z}|
	\] for all $z, \tilde{z} \in \mathbb{R}^D$. Hence, for an optimal coupling $\pi \in \Pi(\kappa, \kappa^m)$ attaining $W_1(\kappa, \kappa^m)$, we get
	\begin{align*}
	\|\Delta_m (f)\|_{\cal H} &\leq \int \| L_{K, \rho_X^z}(f) - L_{K, \rho_X^{\tilde{z}}}(f)\|_{\cal H} \,\pi(dz, d\tilde{z}) \leq 2L W_1(\kappa, \kappa^m),
	\end{align*}
	which implies $\|\Delta_m\|_{\rm op} \rightarrow 0$ for $m\rightarrow \infty$. 
	Thus we can apply stability results for the eigenvalues of $L_{K, \rho_X * \kappa}$ (e.g., \citet[Theorem 1]{chiappinelli2000nonlinear}), showing that $u_i^m \rightarrow \hat{\sigma}_i$ for $i=1,\dots,r$ and $m\rightarrow \infty$.
	Continuing from \eqref{eq:finitesumsineq}, we conclude
	\[
	\sum_{i=1}^r \hat{\sigma}_i \leq \sum_{i=1}^r \sigma_i.
	\]
	Together with \eqref{eq:infsum}, this yields
	\[
	\sum_{i=r+1}^\infty \hat{\sigma}_i \geq \sum_{i=r+1}^\infty \sigma_i.
	\]
	Since $\sigma_n \sim n^{-a}$, assuming $\hat \sigma_n \lesssim n^{-c}$, this inequality yields $c \leq a$, completing the proof.
%
\end{proof}

\begin{proof}\textbf{of Theorem \ref{thm:stab}}
	First, since the embedding $E\colon \mathcal{H} \rightarrow L^2(\rho_3)$ satisfies $E^* E = T_{\rho_3}$, using the polar decomposition for $E$ one obtains
	\[
	\|f_{\lambda, \rho_1} - f_{\lambda, \rho_2}\|_{L^2(\rho_3)} = \|\sqrt{T_{\rho_3}} (f_{\lambda, \rho_1} - f_{\lambda, \rho_2})\|_\mathcal{H},
	\]
	which is shown for instance in \citet[Proof of Proposition 2]{de2005risk}. 
	\citet[Proof of Theorem 4, Step 2.1]{caponnetto2007optimal} have shown by a Neumann series argument, that the assumption from Equation~\eqref{rem:stab} implies 
	\[
	\|\sqrt{T_{\rho_3}}(T_{\rho_1} + \lambda I)^{-1}\|_{\rm op} \leq \frac{1}{\sqrt{\lambda}}.
	\]

	To bound $\|f_{\lambda, \rho_1} - f_{\lambda, \rho_2}\|_{L^2(\rho_3)}$, we use the decomposition
	\[
	f_{\lambda, \rho_1} - f_{\lambda, \rho_2} = (T_{\rho_1} + \lambda I)^{-1}(g_{\rho_1} - g_{\rho_2}) + ((T_{\rho_1} + \lambda I)^{-1} - (T_{\rho_2} + \lambda I)^{-1}) g_{\rho_2}
	\]
	and obtain
	\[
	\|f_{\lambda, \rho_1} - f_{\lambda, \rho_2}\|_{L^2(\rho_3)} \leq R_1 + R_2,
	\]
	where
	\begin{align*}
	R_1 &= \|\sqrt{T_{\rho_3}} (T_{\rho_1} + \lambda I)^{-1}(g_{\rho_1} - g_{\rho_2})\|_{\mathcal{H}}, \\
	R_2 &= \|\sqrt{T_{\rho_3}} ((T_{\rho_1} + \lambda I)^{-1} - (T_{\rho_2} + \lambda I)^{-1}) g_{\rho_2}\|_{\mathcal{H}}.
	\end{align*}
	
	To bound $R_1$, we find
	\[
	R_1 \leq \|\sqrt{T_{\rho_3}}(T_{\rho_1} + \lambda I)^{-1}\|_{\rm op} \|g_{\rho_1} - g_{\rho_2}\|_{\mathcal{H}} \leq \frac{1}{\sqrt{\lambda}} \|g_{\rho_1} - g_{\rho_2}\|_{\mathcal{H}}
	\]
	and for an optimal coupling $\pi \in \Pi(\rho_1, \rho_2)$ that attains $W_1(\rho_1, \rho_2)$ (such an optimal coupling exists by standard results, cf.~\cite[Theorem 4.1]{villani2009optimal}) and abbreviating $\pi(dx,dy)=\pi(d(x_1, y_1), d(x_2, y_2))$, we get
	\begin{align*}
	\|g_{\rho_1} - g_{\rho_2}\|_{\mathcal{H}} &= \Big\|\int K(x_1, \cdot) y_1 \,\rho_1(dx_1, dy_1) - \int K(x_2, \cdot) y_2 \,\rho_2(dx_2, dy_2)\Big\|_{\mathcal{H}} \\
	&= \Big\|\int K(x_1, \cdot) y_1 - K(x_2, \cdot) y_2 \,\pi(dx, dy)\Big\|_{\mathcal{H}} \\
	&\leq \int \|K(x_1, \cdot) y_1 - K(x_2, \cdot) y_2\|_\mathcal{H}  \,\pi(dx, dy)\\
	&\leq \int |y_1| \|K(x_1, \cdot) - K(x_2, \cdot)\|_{\mathcal{H}} + |y_1 - y_2| \|K(x_1, \cdot)\|_{\mathcal{H}} \,\pi(dx, dy)\\
	&\leq \int M L |x_1 - x_2| + |y_1 - y_2| \,\pi(dx, dy)\\
	&\leq \max\{M L, 1\}\,W_1(\rho_1, \rho_2),
	\end{align*}
	which completes the bound for $R_1$.
	
	For $R_2$, we calculate
	\begin{align*}
	R_2 &= \| \sqrt{T_{\rho_3}}\Big( (T_{\rho_1} + \lambda I)^{-1} (T_{\rho_1} - T_{\rho_2}) (T_{\rho_2} + \lambda I)^{-1} \Big) g_{\rho_2}\|_{\mathcal{H}} \\
	&\leq \|\sqrt{T_{\rho_3}} (T_{\rho_1} + \lambda I)^{-1}\|_{\rm op} \|(T_{\rho_1} - T_{\rho_2}) f_{\lambda, \rho_2}\|_{\mathcal{H}} \\
	&\leq \frac{1}{\sqrt{\lambda}} \|(T_{\rho_1} - T_{\rho_2}) f_{\lambda, \rho_2}\|_{\mathcal{H}}.
	\end{align*}
	We choose an optimal coupling $\pi \in \Pi(\rho_1, \rho_2)$ attaining $W_1(\rho_1, \rho_2)$ and denote by $\pi_X$ its marginal distribution on the space $\mathcal{X} \times \mathcal{X}$. We get
	\begin{align}
	\|(T_{\rho_1} - T_{\rho_2}) f_{\lambda, \rho_2}\|_{\mathcal{H}} 
	&= \Big\| \int K(x_1, \cdot) f_{\lambda, \rho_2}(x_1) - K(x_2, \cdot) f_{\lambda, \rho_2}(x_2) \,\pi_X(dx)\Big\|_{\mathcal{H}}\notag\\
	&\leq \int \|(K(x_1, \cdot)-K(x_2, \cdot)) f_{\lambda, \rho_2}(x_1)\|_{\mathcal{H}}\notag\\ 
	&\qquad+ \|(f_{\lambda, \rho_2}(x_1) - f_{\lambda, \rho_2}(x_2)) K(x_2, \cdot)\|_{\mathcal{H}} \,\pi_X(dx)\notag\\
	&\leq \int \|K(x_1,\cdot)-K(x_2,\cdot)\|_{\cal H}|f_{\lambda, \rho_2}(x_1)|\notag\\
	&\qquad+ |f_{\lambda, \rho_2}(x_1)-f_{\lambda, \rho_2}(x_2)|\|K(x_2,\cdot)\|_{\cal H} \pi_X(dx)\label{eq:inproofStab}\\
	&\leq \int L \|f_{\lambda,\rho_2}\|_\infty |x_1 - x_2| + \|f_{\lambda,\rho_2}\|_{\mathrm{Lip}} |x_1 - x_2| \pi_X(dx)\notag\\
	&\leq (L \|f_{\lambda,\rho_2}\|_{\infty} + \|f_{\lambda,\rho_2}\|_{\mathrm{Lip}}) W_1(\rho_1, \rho_2).\notag
	\end{align}
	
	It remains to verify condition $\|(T_{\rho_3} - T_{\rho_1})(T_{\rho_3}+\lambda I)^{-1}\|_{\rm op} \leq \frac{1}{2}$ from Equation~\eqref{rem:stab} based on \eqref{eq:assstab}. We have
	\[
	  \|(T_{\rho_3} - T_{\rho_1})(T_{\rho_3}+\lambda I)^{-1}\|_{\rm op}\le \|T_{\rho_3} - T_{\rho_1}\|_{\rm op}\|(T_{\rho_3}+\lambda I)^{-1}\|_{\rm op}\le \lambda^{-1}\|T_{\rho_3} - T_{\rho_1}\|_{\rm op}.
	\]
	To bound the latter operator norm we use \eqref{eq:inproofStab} and \eqref{eq:KLip} to estimate for any $f\in\mathcal H$ and an optimal coupling $\pi\in\Pi(\rho_1,\rho_3)$ attaining $W_1(\rho_1,\rho_3)$
	\begin{align*}
	\|(T_{\rho_3} - T_{\rho_1}) f\|_{\mathcal{H}} 
	&\leq \int \|K(x_1,\cdot)-K(x_2,\cdot)\|_{\cal H}|\langle K(x_1,\cdot),f\rangle_{\cal H}|\\
	&\qquad + |\langle K(x_1,\cdot)-K(x_2,\cdot),f\rangle_{\cal H}|\|K(x_2,\cdot)\|_{\cal H} \pi_X(dx)\\
	&\leq 2\|f\|_{\cal H}\int \|K(x_1,\cdot)-K(x_2,\cdot)\|_{\cal H}\pi_X(dx)\\
	&\le 2L\|f\|_{\cal H}W_1(\rho_1,\rho_3).
	\end{align*}
    Therefore, $\|T_{\rho_3} - T_{\rho_1}\|_{\rm op}\le 2LW_1(\rho_1,\rho_3)$ and the required condition is satisfied if we have $W_1(\rho_1,\rho_3)\le\frac\lambda{4L}$.
\end{proof}

\subsection{Proofs of Section \ref{sec:secmain}}

\begin{proof}\textbf{of Proposition \ref{prop:pcakernelstability}}
	We apply Theorem \ref{thm:stab} with the assumption given by Equation~\eqref{rem:stab} for $\rho_3 := \tilde{\rho}$ and 
	\begin{align*}
	\rho_1 := \frac{1}{n} \sum_{i=1}^n \delta_{(X_i^*, Y_i)}, ~~~
	\rho_2 := \frac{1}{n} \sum_{i=1}^n \delta_{(\hat{X}_i, Y_i)}.
	\end{align*}
	Hence, $W_1(\rho_1, \rho_2) \leq S_n\|\hat P_n-P\|_{\rm op}$. It remains to verify that assumption \eqref{eq:assnlargeenough} implies condition \eqref{eq:assstab} with high probability. To this end, we proceed as \citet[Proof of Theorem 4, Step 2.1]{caponnetto2007optimal}. We can write
	\[
	  (T_{\rho_3}+\lambda)^{-1}(T_{\rho_3} - T_{\rho_1})=-\frac{1}{n}\sum_{i=1}^n (\Xi_i-\mathbb E[\Xi_i] )
	\]
	for random Hilbert-Schmidt operators 
	\[
	  \Xi_i\colon \mathcal H\to\mathcal H\qquad f\mapsto \Xi_if=(T_{\tilde\rho}+\lambda I)^{-1}T_{\delta_{X_i}} 
    \]
	where $\delta_{X_i}$ denotes the Dirac measure in $X_i$. Since $\|T_{\tilde\rho}\|_{\rm op}\le\int_{\mathcal X}\|T_{\delta_x}\|_{\rm op}\tilde \rho_{X^*}(dx)\le 1$ and thus $\|(T_{\tilde\rho}+\lambda I)^{-1}\|_{\rm op}\le \lambda^{-1}$ we indeed have a bounded Hilbert-Schmidt norm:
	$$\|\Xi_i\|_{\rm HS}\le \|(T_{\tilde\rho}+\lambda I)^{-1}\|_{\rm op}\|T_{\delta_{X_i}}\|_{\rm HS}\le\frac1\lambda\|K(X_i,\cdot)\|_{\mathcal H}^2=\frac1\lambda.$$
	Moreover, 
	\begin{align*}
	  \mathbb E\big[\|\Xi_i\|_{\rm HS}^2\big]
	  &\le \int_{\mathcal X}\|T_{\delta_x}\|_{\rm op}{\rm Tr}\big((T_{\tilde\rho}+\lambda I)^{-2}T_{\delta_x}\big)\tilde \rho_{X^*}(dx)\\
	  &\le {\rm Tr}\big((T_{\tilde\rho}+\lambda I)^{-2}T_{\tilde \rho}\big)\\
	  &\le \|(T_{\tilde\rho}+\lambda I)^{-1}\|_{\rm op}{\rm Tr}\big((T+\lambda I)^{-1}T_{\tilde \rho}\big)
	  \le\frac{\mathcal N(\lambda)}{\lambda}.
	\end{align*}
    Based on these estimates and independence of $\Xi_i$, we can apply a Bernstein inequality for Hilbert space valued random variables \citep[Proposition~2]{caponnetto2007optimal} which yields
    \[
      \big\|(T_{\rho_3}+\lambda)^{-1}(T_{\rho_3} - T_{\rho_1})\big\|_{\rm HS}\le 2\log\Big(\frac{6}{\eta}\Big)\Big(\frac{2}{n\lambda}+\sqrt{\frac{\mathcal N(\lambda)}{\lambda n}}\Big)
    \]
    with probability greater than $1-\eta/3$. The condition on $\mathcal N(\lambda)$ together with $\|T_{\tilde\rho}\|_{\rm op}\ge \lambda$ finally imply that the previous upper bound not larger than 1/2.
\end{proof}

\begin{proof}\textbf{of Theorem \ref{thm:generaldimred}}
	We use the error decomposition \eqref{eq:decomp}. The first term can be bounded by
	\begin{align*}
	  \|(\hat{f}_{\lambda, n} \circ \hat{P}_n - \hat{f}_{\lambda, n} \circ P)\|^2_{L^2(\rho_X)}
	  &\le \|\hat f_{\lambda,n}\|_{\rm Lip}^2\|\hat P_n-P\|_{L^2(\rho_X)}^2\\
	  &\le \|\hat f_{\lambda,n}\|_{\rm Lip}^2\|\hat P_n-P\|_{\rm op}^2\int x^2\rho_X(dx).
	\end{align*}
	The second term can be treated with Proposition~\ref{prop:pcakernelstability}:
	\[
	   \|(\hat{f}_{\lambda, n} \circ P - f_{\lambda, n}^* \circ P)\|^2_{L^2(\rho_X)}\le  \frac {S_n^2  \big(1+LM+\|\hat f_{\lambda,n}\|_{\mathrm{Lip}} + \|\hat{f}_{\lambda, n}\|_\infty \big)^2}{\lambda} \, \big\|\hat{P}_n - P\big\|_{\rm op}^2.
	\]
	while the last term $\|(\boldsymbol{f}^*_{\lambda, n} \circ P - f_{\tilde{\rho}} \circ P)\|^2_{L^2(\rho_X)}=\|\boldsymbol{f}^*_{\lambda, n} - f_{\tilde{\rho}}\|^2_{L^2(\rho_X\circ P^{-1})}$ is the error term for the kernel learning problem for the distribution $\tilde{\rho}$.
\end{proof}

\begin{proof}\textbf{of Lemma \ref{lem:pcaerror}}
	We require some additional notation. Let
	\begin{align*}
	R(P) &:= \int |x -P(x )|^2 \,\rho_X(dx)\qquad\text{and}\qquad  R(\hat{P}_n) := \int |x  -\hat{P}_n(x)|^2 \,\rho_X(dx).
	\end{align*}
	Define $\mathcal{E}_d^{\text{PCA}} := R(\hat{P}_n) - R(P)$.
	Of course, $\hat{P}_n$ is a random variable depending on $(X_1, \dots, X_n)$. By \citet[Equation $2.21$]{reiss2020nonasymptotic} it holds 
	\[
	\mathbb{E}[\|\hat{P}_n - P\|_2^2] \leq \frac{2 \mathbb{E}[\mathcal{E}_d^{\text{PCA}}]}{\sigma^X_{d+1} - \sigma^X_d}.
	\]
	In \citet[Proposition 2, Equation 2.6]{reiss2020nonasymptotic} it is shown that there exists a constant $C_{pca}$ which is independent of $n$ such that
	\[
	\mathbb{E}[\mathcal{E}_d^{\text{PCA}}] \leq \frac{C_{pca}^2}{n (\sigma^X_{d+1} - \sigma^X_d)}
	\]
	and by an application of Markov's inequality, the claim follows.
\end{proof}

\begin{proof}\textbf{of Theorem \ref{thm:overallerror}}
	To apply Theorem \ref{thm:generaldimred}, we use \citet[Theorem 1]{steinwart2009optimal} (c.f.~Lemma \ref{rem:optimalrate}) to obtain 
	\[
	\|(\boldsymbol{f}^*_{\lambda, n} \circ P - f_{\tilde{\rho}} \circ P)\|^2_{L^2(\rho_X)} \leq C_l \log(9/\eta) \Big(\lambda^{\beta} + \frac{1}{\lambda n}\Big)
	\]
	with probably $1-\eta/3$,
	where we note that part (i) of Assumption \ref{ass:alphabeta} always holds with $\alpha=1$ (c.f.~Remark \ref{rem:sobolev}). We further apply Lemma \ref{lem:pcaerror} with $\eta/3$ which yields
	\[
	\|\hat{P}_n-P\|_{\rm op} \leq \frac{\Gamma}{\eta \sqrt{n}}.
	\]
	with probability $1-\eta/3$. With a constant $\bar{C}$ such that
	\[
	\Gamma (1+ LM + \|\hat{f}_{\lambda, n}\|_{\rm Lip} + \|\hat{f}_{\lambda, n}\|_{\infty})^2 \leq  \bar{C}(\|\hat f_{\lambda,n}\|_{\mathrm{Lip}}^2 + \|\hat{f}_{\lambda, n}\|_\infty^2 + 1)
	\]
	Theorem~\ref{thm:generaldimred} yields the claim.
\end{proof}

\begin{proof}\textbf{of Lemma~\ref{lem:normbound}}
 By construction of $\hat{f}_{\lambda,n}$ we have for any $f\in\mathcal{H}$
\[
\frac{1}{n}\sum_{i=1}^{n}(\hat{f}_{\lambda,n}(\hat{X}_{i})-Y_{i})^{2}+\lambda\|\hat{f}_{\lambda,n}\|_{\mathcal{H}}^{2}\le\frac{1}{n}\sum_{i=1}^{n}(f(\hat{X}_{i})-Y_{i})^{2}+\lambda\|f\|_{\mathcal{H}}^{2}.
\]
This gives the basic oracle inequality
\begin{align*}
\frac{1}{n}\sum_{i=1}^{n}\big(\hat{f}_{\lambda,n}(\hat{X}_{i})-f_{\tilde{\rho}}(\hat{X}_{i})\big)^{2}+\lambda\|\hat{f}_{\lambda,n}\|_{\mathcal{H}}^{2} & \le\frac{1}{n}\sum_{i=1}^{n}\big(f(\hat{X}_{i})-f_{\tilde{\rho}}(\hat{X}_{i})\big)^{2}+\lambda\|f\|_{\mathcal{H}}^{2}\\
 & \qquad-\frac{2}{n}\sum_{i=1}^{n}\big(Y_{i}-f_{\tilde{\rho}}(\hat{X}_{i})\big)\big(f(\hat{X}_{i})-\hat{f}_{\lambda,n}(\hat{X}_{i})\big).
\end{align*}
On the event
\[
\mathcal{A}_{n}:=\Big\{\Big\|\frac{1}{n}\sum_{i=1}^{n}\big(Y_{i}-f_{\tilde{\rho}}(\hat{X}_{i})\big)K(\hat{X}_{i},\cdot)\Big\|_{\mathcal{H}}\le\frac{\kappa}{\sqrt{n}}\Big\}
\]
we have
\begin{align*}
\Big|\frac{2}{n}\sum_{i=1}^{n}\big(Y_{i}-f_{\tilde{\rho}}(\hat{X}_{i})\big)\big(f(\hat{X}_{i})-\hat{f}_{\lambda,n}(\hat{X}_{i})\big)\Big| & =\Big|\frac{2}{n}\sum_{i=1}^{n}\big(Y_{i}-f_{\tilde{\rho}}(\hat{X}_{i})\big)\langle K(\hat{X}_{i},\cdot),f-\hat{f}_{\lambda,n}\rangle_{\mathcal{H}}\Big|\\
 & \le\|f-\hat{f}_{\lambda,n}\|_{\mathcal{H}}\Big\|\frac{2}{n}\sum_{i=1}^{n}\big(Y_{i}-f_{\tilde{\rho}}(\hat{X}_{i})\big)K(\hat{X}_{i},\cdot)\Big\|_{\mathcal{H}}\\
 & =\frac{2\kappa}{\sqrt{n}}\|f-\hat{f}_{\lambda,n}\|_{\mathcal{H}}\\
 & \le\frac{2\kappa}{\sqrt{n}}\|f\|_{\mathcal{H}}+\frac{2\kappa}{\sqrt{n}}\|\hat{f}_{\lambda,n}\|_{\mathcal{H}}\\
 & \le\frac{3\kappa}{n\lambda}+\lambda\|f\|_{\mathcal{H}}^{2}+\frac{\lambda}{2}\|\hat{f}_{\lambda,n}\|_{\mathcal{H}}^{2}.
\end{align*}
Therefore, we have on $\mathcal{A}_{n}$ for any $f\in\mathcal{H}$
\[
\frac{1}{n}\sum_{i=1}^{n}\big(\hat{f}_{\lambda,n}(\hat{X}_{i})-f_{\tilde{\rho}}(\hat{X}_{i})\big)^{2}+\frac{\lambda}{2}\|\hat{f}_{\lambda,n}\|_{\mathcal{H}}^{2}\le\frac{1}{n}\sum_{i=1}^{n}\big(f(\hat{X}_{i})-f_{\tilde{\rho}}(\hat{X}_{i})\big)^{2}+\frac{3\kappa}{n\lambda}+2\lambda\|f\|_{\mathcal{H}}^{2}.
\]
If $f_{\tilde{\rho}}\in\mathcal{H}$, we can choose $f=f_{\tilde{\rho}}$
and obtain
\[
\|\hat{f}_{\lambda,n}\|_{\mathcal{H}}^{2}\le\frac{2}{n\lambda}\sum_{i=1}^{n}\big(\hat{f}_{\lambda,n}(\hat{X}_{i})-f_{\tilde{\rho}}(\hat{X}_{i})\big)^{2}+\|\hat{f}_{\lambda,n}\|_{\mathcal{H}}^{2}\le\frac{6\kappa}{n\lambda^{2}}+4\|f_{\tilde{\rho}}\|_{\mathcal{H}}^{2}.
\]
For $\lambda= n^{-1/2}$ (which is the optimal choice for $\alpha=\beta=1$)
we conclude $\|\hat{f}_{\lambda,n}\|_{\mathcal{H}}^{2}\le6\kappa+4\|f_{\tilde{\rho}}\|_{\mathcal{H}}^{2}$
on the event $\mathcal{A}_{n}$. 

It remains to verify $\mathbb P(\mathcal{A}_{n}^{c})\to0$. To this end,
we exploit $f_{\tilde\rho}(X_i^*)=\mathbb E[Y_i|X_i^*]$ such that we can write $Y_{i}=f_{\tilde{\rho}}(X_{i}^{*})+\delta_{i}$ with $\mathbb E[\delta_i|X_i^*]=0$ and $|\delta_i|\le 2M$. The decomposition $Y_{i}-f_{\tilde{\rho}}(\hat{X}_{i})=\delta_{i}+f_{\tilde{\rho}}(X_{i}^{*})-f_{\tilde{\rho}}(\hat{X}_{i})$
yields
\begin{align*}
\Big\|\frac{1}{n}\sum_{i=1}^{n}\big(Y_{i}-f_{\tilde{\rho}}(\hat{X}_{i})\big)K(\hat{X}_{i},\cdot)\Big\|_{\mathcal{H}} 
& \le\Big\|\frac{1}{n}\sum_{i=1}^{n}\delta_{i}K(X^*_{i},\cdot)\Big\|_{\mathcal{H}}+\Big\|\frac{1}{n}\sum_{i=1}^{n}\delta_{i}\big(K(\hat{X}_{i},\cdot)-K(X^*_{i},\cdot)\big)\Big\|_{\mathcal{H}}\\
&\qquad +\Big\|\frac{1}{n}\sum_{i=1}^{n}\big(f_{\tilde{\rho}}(X_{i}^{*})-f_{\tilde{\rho}}(\hat{X}_{i})\big)K(\hat{X}_{i},\cdot)\Big\|_{\mathcal{H}}\\
& =:T_{1}+T_2+T_{2}.
\end{align*}
In particular, we have $\mathbb P(\mathcal{A}_{n}^{c})\le\mathbb P(T_{1}\ge\frac{\kappa}{3\sqrt{n}})+\mathbb P(T_{1}\ge\frac{\kappa}{3\sqrt{n}})+\mathbb P(T_{3}\ge\frac{\kappa}{3\sqrt{n}})$. 

The first probability can be bounded with Markov's inequality and
using RKHS properties:
\begin{align*}
\mathbb P(T_{1}\ge\frac{K}{2\sqrt{n}}) & \le\frac{4n}{\kappa^{2}}\mathbb E\Big[\Big\|\frac{1}{n}\sum_{i=1}^{n}\delta_{i}K(X^*_{i},\cdot)\Big\|_{\mathcal{H}}^{2}\Big]\\
 & =\frac{4n}{\kappa^{2}}\mathbb E\Big[\frac{1}{n^{2}}\sum_{i,j=1}^{n}\delta_{i}\delta_{j}\langle K(X^*_{i},\cdot),K(X^*_{j},\cdot)\rangle_{\mathcal{H}}\Big]
  =\frac{4}{\kappa^{2}n}\sum_{i,j=1}^{n}\mathbb E\Big[\delta_{i}\delta_{j}K(X^*_{i},X^*_{j})\Big].
\end{align*}
Since $(\delta_{i})$ are independent and satisfy $\mathbb E[\delta_{i}|X^*_{i}]=0,\mathbb E[\delta_{i}^{2}|X^*_{i}]=:\sigma^2\le4M^2$,
we obtain 
\[
\mathbb P(T_{1}\ge\frac{K}{3\sqrt{n}})
\le\frac{9\sigma^{2}}{\kappa^{2}n}\sum_{i=1}^{n}\mathbb E\big[K(X^*_{i},X^*_{i})\big]
=\frac{9\sigma^{2}}{\kappa^{2}}\mathbb E\big[K(X^*_{1},X^*_{1})\big]\le\eta
\]
for $\kappa^{2}>\frac{9\sigma^{2}}{\eta}\mathbb E\big[K(X^*_{1},X^*_{1})\big].$

For the second probability, we use $|\delta_i|\le 2M$ and \eqref{eq:KLip} to bound 
\begin{align*}
 T_2\le \frac{1}{n}\sum_{i=1}^n|\delta_i|\big\|K(\hat X_i,\cdot)-K(X^*_i,\cdot)\big\|_{\mathcal H}
 &\le \frac{2ML}{n}\sum_{i=1}^n\big|\hat X_i-X_i^*\big|\\
 &\le \frac{2ML}{n}\sum_{i=1}^n\big|(\hat P-P)X_i\big|\\
 &\le 2ML\|\hat P-P\|_{\rm op}\Big(\frac{1}{n}\sum_{i=1}^n|X_i|^2\Big)^{1/2}.
\end{align*}
By Lemma~\ref{lem:pcaerror} we have $\mathbb P\Big(\|\hat{P}_{n}-P\|_{{\rm op}}\ge\frac\Gamma{3\eta\sqrt n}\Big)\le\eta$ with $\Gamma$ from Theorem~\ref{thm:overallerror}. Hence,
\begin{align*}
\mathbb P(T_{2}\ge\frac{\kappa}{3\sqrt{n}}) & \le\mathbb P\Big(\frac{2ML\Gamma}{\eta}\Big(\frac{1}{n}\sum_{i=1}^{n}\big|X_{i}\big|^{2}\Big)^{1/2}\ge\kappa\Big)+\eta\\
 & \le\frac{4(ML\Gamma)^2}{\eta^2\kappa^2}\mathbb E\big[\big|X_{1}\big|^{2}\big]+\eta\le2\eta
 \end{align*}
for $\kappa^2>\eta^{-3}4(ML\Gamma)^2\mathbb E[|X_1|^2]$.

For the third probability, we estimate
\begin{align*}
T_{3} & \le\frac{1}{n}\sum_{i=1}^{n}\big|f_{\tilde{\rho}}(X_{i}^{*})-f_{\tilde{\rho}}(\hat{X}_{i})\big|\|K(\hat{X}_{i},\cdot)\|_{\mathcal{H}}\\
 & \le\|f_{\tilde{\rho}}\|_{\mathrm{Lip}}\sqrt{K(\hat{X}_{i},\hat{X}_{i})}\frac{1}{n}\sum_{i=1}^{n}\big|X_{i}^{*}-\hat{X}_{i}\big|
  =\|f_{\tilde{\rho}}\|_{\mathrm{Lip}}\frac{1}{n}\sum_{i=1}^{n}\big|X^*_{i}-\hat X_i\big|.
\end{align*}
From here we can proceed as with $T_2$ and conclude 
\begin{align*}
\mathbb P(T_{3}\ge\frac{\kappa}{3\sqrt{n}}) & \le\mathbb P\Big(\frac{\|f_{\tilde\rho}\|_{\rm Lip}\Gamma}{\eta}\Big(\frac{1}{n}\sum_{i=1}^{n}\big|X_{i}\big|^{2}\Big)^{1/2}\ge\kappa\Big)+\eta\\
 & \le\frac{\|f_{\tilde\rho}\|_{\rm Lip}^2\Gamma^2}{\eta^2\kappa^2}\mathbb E\big[\big|X_{1}\big|^{2}\big]+\eta\le2\eta
 \end{align*}
for $\kappa^2>\eta^{-3}(\|f_{\tilde\rho}\|_{\rm Lip}\Gamma)^2\mathbb E[|X_1|^2]$.
Therefore, there is a constant $D>0$ such that $\mathbb P(\mathcal{A}_{n}^{c})\ge1-5\eta$
for $\kappa=D\eta^{-3/2}$.
\end{proof}

\begin{proof}\textbf{of Theorem \ref{prop:mn}}
	The main inequality \eqref{eq:supervisederror} follows completely analogously to Theorem \ref{thm:overallerror}. The optimal choices of $\lambda$ in the given regimes follows directly since for $n > m^{\alpha/(\beta + 1)}$ and $\lambda = m^{-\frac{1}{1+\beta}}$, it holds $\frac{1}{\lambda m} > \frac{1}{\lambda^\alpha n}$. And on the other hand for $m > n^{\frac{1+\beta}{\beta+\alpha}}$ and $\lambda = n^{-\frac{1}{\beta+\alpha}}$ it holds $\frac{1}{\lambda m} < \frac{1}{\lambda^\alpha n}$.
\end{proof}

\begin{proof}\textbf{of Proposition \ref{prop:procedureerror}}
	We decompose
	\[
	\| f_\rho - f_{\tilde{\rho}} \circ P \|_{L^2(\rho_X)}
	\le \| f_\rho - f_{\rho} \circ P \|_{L^2(\rho_X)}
	+\| f_\rho\circ P - f_{\tilde{\rho}} \circ P \|_{L^2(\rho_X)}.
	\]
	For the second term, we note that 
	\[
	f_{\tilde \rho}\big(P(X)\big)=\E{Y|P(X)}
	=\E{\E{Y|X}\big|P(X)}
	=\E{f_\rho(X)|P(X)}.
	\]
	Therefore,
	\begin{align*}
	\| f_\rho\circ P - f_{\tilde{\rho}} \circ P \|_{L^2(\rho_X)}^2
	&=\E{\big(f_\rho(P(X))-\E{f_\rho(X)|P(X)}\big)^2}\\
	&=\E{\E{f_\rho(P(X))-f_\rho(X)\big|P(X)}^2}\\
	&\le \E{(f_\rho(P(X))-f_\rho(X))^2}\\
	&=\| f_\rho - f_{\rho} \circ P \|_{L^2(\rho_X)}^2.
	\end{align*}
	We obtain
	\begin{align*}
	\| f_\rho - f_{\tilde{\rho}} \circ P \|_{L^2(\rho_X)}
	&\le 2\| f_\rho - f_{\rho} \circ P \|_{L^2(\rho_X)}\\
	&\le 2L\E{|X-P(X)|^2}^{1/2}
	\le 2L\Big(\sum_{i=d+1}^D \sigma_i^X\Big)^{1/2}.
	\end{align*}
	where the last inequality is the well known reconstruction error for PCA, see, e.g., \cite{reiss2020nonasymptotic}. 
\end{proof}


\acks{Research supported by the Landesforschungsf\"{o}rderung Hamburg under the project LD-SODA. MT acknowledges finanical support by the DFG through the grant TR 1349-3.}


\vskip 0.2in
\bibliography{dre}

\end{document}